\documentclass{article}

\usepackage{microtype}
\usepackage{graphicx}
\usepackage{booktabs} 

\usepackage{hyperref}

\usepackage[accepted]{icml2024}

\usepackage{amsmath}
\usepackage{amssymb}
\usepackage{mathtools}
\usepackage{amsthm}
\usepackage{enumitem}
\usepackage[capitalize,noabbrev]{cleveref}

\theoremstyle{plain}
\newtheorem{theorem}{Theorem}[section]

\newtheorem{lemma}[theorem]{Lemma}
\newtheorem{corollary}[theorem]{Corollary}
\theoremstyle{definition}
\newtheorem{definition}[theorem]{Definition}

\theoremstyle{remark}
\newtheorem{remark}[theorem]{Remark}

\usepackage[textsize=tiny]{todonotes}

\DeclareMathOperator*{\argmax}{arg\,max}

\usepackage{subcaption}

\icmltitlerunning{Federated Combinatorial Optimization with Multi-Agent Multi-Armed Bandits}

\begin{document}

\twocolumn[







\icmltitle{Federated Combinatorial Multi-Agent Multi-Armed Bandits}




\begin{icmlauthorlist}
\icmlauthor{Fares Fourati}{yyy}
\icmlauthor{Mohamed-Slim Alouini}{yyy}
\icmlauthor{Vaneet Aggarwal}{comp}
\end{icmlauthorlist}

\icmlaffiliation{yyy}{Department of Computer, Electrical and Mathematical Science and Engineering, King Abdullah University of Science and Technology (KAUST), Thuwal, KSA.}
\icmlaffiliation{comp}{School of Industrial Engineering, Purdue University, West
Lafayette, IN 47907, USA}

\icmlcorrespondingauthor{Fares Fourati}{fares.fourati@kaust.edu.sa}

\icmlkeywords{Machine Learning, ICML}

\vskip 0.3in
]



\printAffiliationsAndNotice{}  

\begin{abstract}
This paper introduces a federated learning framework tailored for online combinatorial optimization with bandit feedback. In this setting, agents select subsets of arms, observe noisy rewards for these subsets without accessing individual arm information, and can cooperate and share information at specific intervals.
Our framework transforms any offline resilient single-agent $(\alpha-\epsilon)$-approximation algorithm—having a complexity of $\tilde{\mathcal{O}}\left(\frac{\psi}{\epsilon^\beta}\right)$, where the logarithm is omitted, for some function $\psi$ and constant $\beta$—into an online multi-agent algorithm with $m$ communicating agents and an $\alpha$-regret of no more than $\tilde{\mathcal{O}}\left(m^{-\frac{1}{3+\beta}} \psi^\frac{1}{3+\beta} T^\frac{2+\beta}{3+\beta}\right)$.
Our approach not only eliminates the $\epsilon$ approximation error but also ensures sublinear growth with respect to the time horizon $T$ and demonstrates a linear speedup with an increasing number of communicating agents. Additionally, the algorithm is notably communication-efficient, requiring only a sublinear number of communication rounds, quantified as $\tilde{\mathcal{O}}\left(\psi T^\frac{\beta}{\beta+1}\right)$. 
Furthermore, the framework has been successfully applied to online stochastic submodular maximization using various offline algorithms, yielding the first results for both single-agent and multi-agent settings and recovering specialized single-agent theoretical guarantees. We empirically validate our approach to a stochastic data summarization problem, illustrating the effectiveness of the proposed framework, even in single-agent scenarios.
\end{abstract}

\section{Introduction}
\label{intro}

The \textit{Multi-Armed Bandits} (MAB) \cite{slivkins2019introduction, lattimore2020bandit} model online decision-making, where at every time step, an agent plays an arm and observes its associated reward. In combinatorial MAB, the agent can play a set of arms, instead of one arm, at each time step and receive a reward for that selection. When the agent only learns about the reward linked to the played set, it is known as \textit{full-bandit feedback} or \textit{bandit feedback}. If the agent gains additional insights into how each arm contributes to the overall reward, it is called \textit{semi-bandit feedback}. Dealing with bandit feedback is more challenging since agents have less knowledge than in the semi-bandit feedback setting. In this work, we consider bandit feedback \cite{fourati2023randomized, nie2023framework, fourati2024combinatorial}, which has several applications, such as recommender systems, revenue maximization \cite{fourati2023randomized}, influence maximization \cite{nie2023framework, fourati2024combinatorial}, and data summarization, as shown in this work. 

Federated learning (FL), an emerging machine learning paradigm, involves collaborative learning among multiple agents. In this process, selected agents share their local updates with a central server, which then aggregates these updates and sends the consolidated output back to each participating agent \cite{konevcny2016federated, mcmahan2017communication, li2020federated, hosseinalipour2020federated, elgabli2022fednew, fourati2023filfl}. While the problem has been studied in the context of continuous optimization, we address combinatorial optimization \citep{korte2011combinatorial} in a federated online stochastic setting with bandit feedback. For example, the multi-agent setting under consideration can be applied to recommender systems, where each agent aims to recommend a set of products and then shares its findings with a server. We provide a general FL framework to adapt combinatorial offline single-agent approximation algorithms to online multi-agent settings, presenting the first results for the regret bounds in this setup.

We consider a setting with $m' \geq 1$ agents connected through a network. Among them, only a randomly selected subset of $m \leq m'$ agents can cooperate and share information in communication rounds, possibly through a server. This setup accommodates scenarios of partial participation, which are more practical in some real-world settings due to inherent availability and communication constraints. When $m = m'$, the scenario reverts to a full-participation setting. All agents aim to solve the same online stochastic combinatorial problem within a time horizon of $T$. They conduct local exploration and then, if selected, share local estimations. Each agent $i$ can play any subset of arms in parallel with others and receives noisy rewards for that set. In a combinatorial setting, the number of possible actions becomes exponentially large with the number of base arms, making sharing estimations for all possible actions prohibitive. Therefore, we consider a more practical approach wherein, in each communication round, selected agents share only a single action (subset) estimation.

Our work is the first to provide a general multi-agent framework for adapting combinatorial offline single-agent approximation algorithms to a FL setting to solve stochastic combinatorial multi-agent MAB (C-MA-MAB) problems with only bandit feedback. The proposed approach provably achieves a regret of at most $\tilde{\mathcal{O}}(m^{-\frac{1}{3+\beta}} \psi^\frac{1}{3+\beta} T^{\frac{2+\beta}{3+\beta}})$, which is sub-linear in the time horizon $T$ and decreases with an increasing number of agents $m$, for some constant $\beta$ and some function $\psi$ that govern the complexity of the considered offline approximation algorithm. The framework does not require the agents to communicate every step; instead, it only needs $\tilde{\mathcal{O}}(\psi T^\frac{\beta}{\beta+1})$ communication times. Our proposed framework can serve to solve online combinatorial problems for both single-agent $(m^{'} = m  = 1)$ and multi-agent $(m^{'} >1)$ scenarios. Notably, our framework enjoys a linear speedup with an increasing number of agents, translating to decreased regrets.

We note that for a single agent, under bandit feedback, various offline-to-online transformations have been proposed for submodular maximization \cite{nie2022explore, fourati2023randomized, fourati2024combinatorial}. Additionally, \citet{nie2023framework} studied a framework for general combinatorial optimization in which any resilient offline algorithm with an $\alpha$-approximation guarantee can be adapted to an online algorithm with sublinear $\alpha$-regret guarantees. Similarly, an offline resilient $(\alpha-\epsilon)$-approximation algorithm can provide sublinear $(\alpha-\epsilon)$-regret guarantees for an online algorithm for any $\epsilon \geq 0$. This approach leads to a linear $\alpha$-regret online algorithm when $\epsilon > 0$. Recently, \citet{fourati2024combinatorial} addressed this issue of linear regret in the case of monotone submodular optimization with a cardinality constraint. They adapted a sub-optimal offline $(1-1/e-\epsilon)$-approximation algorithm—whose complexity grows with $\log(\frac{1}{\epsilon})$—to a sub-linear $(1-1/e)$-regret online algorithm with bandit feedback, successfully eliminating the $\epsilon$ approximation error while ensuring sub-linearity with respect to the horizon $T$.

In this work, we generalize and extend the results previously established by \citet{nie2023framework} and \citet{fourati2024combinatorial}. We demonstrate that any sub-optimal single-agent offline approximation algorithm, with an approximation factor of $(\alpha-\epsilon)$ where $\epsilon \geq 0$, and ensuring resilience, can be adapted to a multi-agent setting with bandit feedback. This adaptation provides sub-linear $\alpha$-regret guarantees, in contrast to the sub-linear $(\alpha-\epsilon)$-regret guarantees provided by \citet{nie2023framework}. This approach eliminates the $\epsilon$ error while maintaining sub-linearity with respect to $T$ across any combinatorial problem, any reward, and under any constraints.

We note that in contrast to previous works that dealt with specific assumptions about reward functions and constraints on actions—such as assuming stochastic submodular rewards \cite{fourati2023randomized} or monotone rewards with cardinality constraints \cite{nie2022explore, fourati2024combinatorial}—our work considers general reward functions without making any assumptions about the reward type or constraints. Furthermore, our work explores the use of any offline algorithm $\mathcal{A}(\epsilon)$ as a subroutine with complexity in the general form of $\mathcal{O}\left(\frac{\psi}{\epsilon^\beta}\right)$ or $\mathcal{O}\left(\frac{\psi}{\epsilon^\beta}\log\left(\frac{1}{\epsilon}\right)\right)$, where $\psi$ is a function of the problem characteristics, $\beta \geq 0$ some constant, and $\epsilon$ is an approximation error factor, ensuring that our results apply to any algorithm with these characteristics.

In addition, unlike previous works that focused on combinatorial single-agent scenarios \cite{fourati2023randomized, fourati2024combinatorial, nie2022explore, nie2023framework}, we address a combinatorial multi-agent setting where collaboration is permitted, possibly with partial participation, and optimize the worst-case regret for any agent, thereby recovering and generalizing the single-agent setting. Additionally, although our proposed algorithm has parameters, such as $\epsilon^\star$ and $r^\star$, none are considered hyperparameters. We derive closed-form values for these parameters as functions of the problem parameters, such as the range $T$, the number of available agents $m$, and the subroutine-related value of the constant $\beta$, and the function $\psi$, to minimize the expected cumulative $\alpha$-regret. 

\textbf{Contributions:}  We introduce a novel FL framework for online combinatorial optimization, adapting single-agent offline algorithms to tackle online multi-agent problems with bandit feedback. The paper demonstrates the adaptability of any single-agent sub-optimal offline $(\alpha-\epsilon)$-approximation algorithm, ensuring resilience, for a C-MA-MAB with $\alpha$-regret guarantees of $\tilde{\mathcal{O}}(m^{-\frac{1}{3+\beta}} T^\frac{2+\beta}{3+\beta})$, lifting the $\epsilon$ error, providing sub-linearity in the horizon $T$, and ensuring linear speedup with an increasing number of agents $m$. The framework only requires sub-linear communication rounds of $\tilde{\mathcal{O}}\left(T^\frac{\beta}{\beta+1}\right)$, which becomes at most logarithmic with respect to $T$ when $\beta=0$. Furthermore, we leverage our theoretical results to address online submodular maximization, present the first specialized regret for the non-monotone cardinality constraint case, and recover bounds for different offline algorithms, demonstrating tighter regrets than previous specialized works. Finally, we showcase the applicability of our framework to both single-agent and multi-agent stochastic data summarization problems against MAB baselines, highlighting its effectiveness in practical scenarios.

\begin{table*}[t]
\small
\centering
\begin{tabular}{|c|c|c|c|c|c|}
\hline 

\begin{tabular}[c]{@{}c@{}} Offline Algorithm  \end{tabular}   &  \begin{tabular}[c]{@{}c@{}} Offline \\ Complexity \end{tabular}   &   \begin{tabular}[c]{@{}c@{}} Offline \\ Factor \end{tabular}                                                             & \begin{tabular}[c]{@{}c@{}}Online \\ Factor $\alpha$\end{tabular} & Prior $\alpha$-regret                                                                                              & \begin{tabular}[c]{@{}c@{}} Our $\alpha$-regret Bound\end{tabular}                                                       \\ \hline 
\begin{tabular}[c]{@{}c@{}} \textsc{RandomizedUSM} \\ \cite{buchbinder2015tight} \end{tabular} & $\mathcal{O}\left(n\right)$  & $1/2$  & $1/2$                                                       & $\tilde{\mathcal{O}}\left(nT^\frac{2}{3}\right)^{*} $                            & $\tilde{\mathcal{O}}\left(m^{-\frac{1}{3}}nT^\frac{2}{3}\right)$                                       
                                              \\ \hline
\begin{tabular}[c]{@{}c@{}} \textsc{Greedy} \\ \cite{nemhauser1978analysis}   \end{tabular} & $\mathcal{O}\left(nk\right)$ & $1-\frac{1}{e}$ & $1-\frac{1}{e}$                                                    & $\tilde{\mathcal{O}}\left(k n^\frac{1}{3} T^\frac{2}{3}\right)^{**}$         & $\tilde{\mathcal{O}}\left(m^{-\frac{1}{3}}k n^\frac{1}{3} T^\frac{2}{3}\right)$  
\\ \hline

\begin{tabular}[c]{@{}c@{}} \textbf{\textsc{General}} \\ \small{Excluding the next rows}\end{tabular}  & $\mathcal{O}(\psi)$  & $\alpha$ &    $\alpha$  &     $\Tilde{\mathcal{O}}\left(\delta^\frac{2}{3}\psi^\frac{1}{3} T^\frac{2}{3}\right)^{**}$      & $\Tilde{\mathcal{O}}\left(m^{-\frac{1}{3}} \delta^\frac{2}{3}\psi^\frac{1}{3}  T^\frac{2}{3}\right)$     
\\ \hline

\begin{tabular}[c]{@{}c@{}} \textsc{Stochastic-Greedy} \\ \small{\cite{mirzasoleiman2015lazier}}   \end{tabular} & $\mathcal{O}\left(n \log(\frac{1}{\epsilon})\right)$ &$1-\frac{1}{e}-\epsilon$ & $1-\frac{1}{e}$      & $\tilde{\mathcal{O}}\left(k^\frac{2}{3} n^\frac{1}{3} T^\frac{2}{3}\right)^{\dagger} $                                                    & $\tilde{\mathcal{O}}\left(m^{-\frac{1}{3}} k^\frac{2}{3} n^\frac{1}{3} T^\frac{2}{3}\right)$      
\\ \hline

\begin{tabular}[c]{@{}c@{}} \textbf{\textsc{More General}} \\ \small{Excluding the next rows} \end{tabular} & $\mathcal{O}(\psi \log^{\gamma}(\frac{1}{\epsilon}))$ &      $\alpha - \epsilon$  &    $\alpha$ & \textbf{None}   & $\Tilde{\mathcal{O}}\left(m^{-\frac{1}{3}} \delta^\frac{2}{3}\psi^\frac{1}{3}  T^\frac{2}{3}\right)$                          
\\ \hline
\begin{tabular}[c]{@{}c@{}} \textsc{RandomSampling} \\ \cite{buchbinder2017comparing}   \end{tabular} & $\mathcal{O}\left(\frac{n}{\epsilon^2} \log(\frac{1}{\epsilon})\right)$ &  $\frac{1}{e}-\epsilon$  & $\frac{1}{e}$        & \textbf{None}                                                   & $\Tilde{\mathcal{O}}\left(m^{-\frac{1}{5}} k^\frac{2}{5} n^\frac{1}{5}  T^\frac{4}{5}\right)$                                                                                         
                  
\\ \hline
\begin{tabular}[c]{@{}c@{}} \textbf{\textsc{Most General}} \\ \small{Including the previous rows} 
 \end{tabular} & $\mathcal{O}( \frac{\psi}{\epsilon^\beta} \log^{\gamma}(\frac{1}{\epsilon}))$  & $\alpha - \epsilon$ &   $\alpha$ &     \textbf{None}    & $\Tilde{\mathcal{O}}\left(m^{-\frac{1}{3+\beta}} \delta^\frac{2}{3+\beta}\psi^\frac{1}{3+\beta}  T^\frac{2+\beta}{3+\beta}\right)$                         
\\ \hline
\end{tabular}
\caption{\small The table summarizes the results of combinatorial optimization under bandit feedback. 
We use $\tilde{\mathcal{O}}$ to simplify expressions. Key parameters include horizon $T$, number of communicating agents $m$, base arm count $n$, and cardinality constraint $k$. Each row presents a specific offline algorithm or a class transformation with a given complexity, offline approximation factor, and a target online factor $\alpha$. For the general rows, we consider classes of offline algorithms, with an approximation error factor $\epsilon$, with general complexity forms, with general constants, $\psi \geq 0$, $\beta \geq 0$, $\gamma \in \{0,1\}$, and $\delta \geq 0$. $^*$ \cite{fourati2023randomized}, $^{**}$ \cite{nie2023framework}, $^{\dagger}$ \cite{fourati2024combinatorial}. 
}
\label{tab:related:work}
\end{table*}

\section{Problem Statement}
\label{prob_state}

We formally present the problem as follows. We denote $\Omega$ as the ground set of $n$ base arms. We consider a set $\mathbb{A}'$ with $m' \geq 1$ agents, where only a randomly selected subset $\mathbb{A}$ of $m \geq 1$ agents communicates in the communication rounds. We examine decision-making problems within a fixed period $T$, where, at every time step $t$, agent $i$ chooses a subset $S_{i,t} \subseteq \Omega$. Let $\mathbb{S} \subseteq 2^\Omega$ represent the set of all permitted subsets, depending on the problem constraints. 

In each time step $t$, agent $i$ plays an action $S_{i,t} \in \mathbb{S}$ and acquires a noisy reward $f_t(S_{i,t})$. We assume that the reward $f_t$ is a realization of a stochastic function with a mean of $f$, bounded in $[0,1]$ \footnote{Results can be directly extended to a general function $f(\cdot)$ with a minimum value $f_{min}$ and a maximum value $f_{max}$ by considering a normalized function $g(\cdot) = (f(\cdot) - f_{min})/(f_{max}- f_{min})$.}, and i.i.d. conditioned on a given action. Thus, over a horizon $T$, agent $i$ achieves a cumulative reward of $\sum_{t=1}^Tf_t(S_{i,t})$. We define the expected reward function for a given action $S$ as 
$f(S) = \mathbb{E}[f_t(S)]$, hence $S^{\star}=\argmax_{S\subseteq \mathbb{S}}f(S)$ denote the optimal set in expectation. 

In offline settings, attention is on the algorithm's complexity and worst-case expected output approximation guarantees. Conversely, in the online setting, attention is on cumulative rewards, where agents seek to minimize their expected cumulative regrets over time. One standard metric to assess the algorithm's online performance is to contrast the agent to an ideal learner that knows and consistently plays the best choice in expectation $S^{\star}$. However, the significance of such a comparison becomes questionable if the optimization of $f$ over $\mathbb{S}$ is NP-hard, if the horizon is not exponentially large in the problem parameters \citep{fourati2023randomized, fourati2024combinatorial, nie2023framework}. Hence, if a polynomial time offline algorithm $\mathcal{A}(\epsilon)$, happen to be an $(\alpha-\varepsilon)$-approximation\footnote{An algorithm $\mathcal{A}(\epsilon)$ is an $(\alpha-\epsilon)$-approximation algorithm for maximizing a deterministic function $f:\mathbb{S}\to \mathbb{R}$ after $N$ oracle calls to $f$ satisfies $\mathbb{E}[f(\Theta)]\geq (\alpha-\epsilon) f(S^{\star})$, where $S^{\star}$ the optimal subset under $f$ and the expectation is over the randomness of $\mathcal{A}(\epsilon)$.} algorithm, for a given error $\epsilon \geq 0$, and an approximation ratio of $(\alpha-\epsilon)\leq 1$, for specific combinatorial objectives, a common approach involves comparing the agent's cumulative reward to $\sum_{t=1}^T (\alpha-\epsilon) f(S^{\star})$ and denoting the difference as the $(\alpha-\epsilon)$-regret \citep{nie2023framework}. In this work, we compare the learner's cumulative reward to a (tighter) agent that achieves $\sum_{t=1}^T \alpha f(S^{\star})$, and we denote the difference as the $\alpha$-regret, which is defined for every agent $i$ as follows:
\begin{equation}
    \mathcal{R}_i(T) =  \sum_{t=1}^T (\alpha f_t(S^{\star}) - f_t(S_{i,t})). \label{eq:reg:1e}
\end{equation}  
The cumulative $\alpha$-regret $\mathcal{R}_i(T)$ is random, given that it is a sum of $\alpha$-regrets, which are functions of stochastic rewards and depend on the chosen subsets. In this work, we aim to minimize the expected cumulative $\alpha$-regret, where the expectation encompasses the oracle noise and the randomness of the series of actions.

\section{Related Work} 

In Table~\ref{tab:related:work}, we provide a comprehensive overview comparing our $\alpha$-regret guarantees with existing ones across various adaptations of offline approximations for different combinatorial problems. Each row corresponds to a case involving an offline algorithm or a general class of algorithms with a specified complexity form, offline approximation factor, and target online regret factor $\alpha$. The last two columns present previously established $\alpha$-regret bounds and our own for adapting the proposed offline algorithm or class of algorithms as a subroutine to the online setting. The third row presents the \textsc{General} transformation, which generalizes the previous two cases, with a complexity of the form $\mathcal{O}(\psi)$ for any function $\psi$. The fifth row presents the \textsc{More General} case, with a complexity of the form $\mathcal{O}(\psi \log^{\gamma}(\frac{1}{\epsilon}))$, where $\gamma \in \{0,1\}$, further generalizing the previous cases. The last row, presenting the \textsc{Most General}, encompasses all previous rows, with a complexity of the form $\mathcal{O}(\frac{\psi}{\epsilon^\beta} \log^{\gamma}(\frac{1}{\epsilon}))$, where $\beta \geq 0$, including the \textsc{General} when $\beta=\gamma=0$, the \textsc{More General} when $\beta=0$, and the example in the sixth row.

\textbf{Combinatorial Single-Agent Examples:} We note that online submodular optimization with bandit feedback has been considered in \cite{nie2022explore, fourati2023randomized, fourati2024combinatorial}. These results are differentiated from our work in Table \ref{tab:related:work}. For single-agent non-monotone submodular rewards (See Section \ref{sec:appl-CETC:submod}) under bandit feedback, \citet{fourati2023randomized} adapts the \textsc{RandomizedUSM} in \cite{buchbinder2015tight}, achieving sub-linear $\frac{1}{2}$-regret. Additionally, for single-agent monotone submodular rewards under bandit feedback with a cardinality constraint $k$, \citet{nie2022explore} adapts the \textsc{Greedy} in \cite{nemhauser1978analysis}, and \citet{fourati2024combinatorial} adapts the \textsc{Stochastic-Greedy} in \cite{mirzasoleiman2015lazier}, with both achieving sub-linear $(1-\frac{1}{e})$-regret. Our work not only recovers all the results above in the single-agent setting but also generalizes them to the multi-agent setting, showing a decreasing regret with a factor of $m^{-\frac{1}{3}}$.

\textbf{Multi-Agent:} Previous works have proposed solutions for solving offline distributed submodular maximization. For example, partitioning the arms among the agents and running a greedy algorithm on each agent was proposed in previous works \cite{barbosa2015power, mirzasoleiman2013distributed}. While this is practical in some settings, it is designed for deterministic objectives and leads to lower approximation guarantees (half the ratio for monotone submodular maximization). 
Moreover, regret analysis for multi-agent (non-combinatorial) MAB problems has been investigated \cite{chawla20Gossiping,wang2020optimal,agarwal2022multi}. In \cite{chawla20Gossiping},  gossip style communication approach is used between agents to achieve ${\tilde{\mathcal{O}}}((\frac{n}{m}+2)^{\frac{1}{3}}T^{\frac{2}{3}})$ regret. \citet{wang2020optimal} present an algorithm for a distributed bandit setting where all the agents communicate with a central node, which is shown to achieve a regret of $\tilde{\mathcal{O}}(\sqrt{nT/m})$. \citet{agarwal2022multi} propose another algorithm, which splits the arms among the different agents, such that each
learner plays arms only within a subset of arms and the best-communicated arm indices from other agents in the previous round. This achieves a regret of ${\tilde{\mathcal{O}}}(\sqrt{(\frac{n}{m}+m)T})$ while reducing the communication significantly. We note that these works do not directly extend to combinatorial bandits since the confidence-bound based approaches here cannot work for combinatorial bandits since $\mathcal{O}(2^n)$ sets cannot be explored. Recent works have considered FL for contextual bandits \cite{li2022asynchronous, he2022simple, li2023learning}; however, these works do not apply to our setting. Our work is the first to present an FL framework for general combinatorial multi-agent MAB with bandit feedback.  

\textbf{Combinatorial Single-Agent Frameworks:} Previous frameworks have been proposed for combinatorial single-agent MAB problems \cite{niazadeh2021online, nie2023framework}. \citet{niazadeh2021online} proposes a framework aiming to adjust an iterative greedy offline algorithm into an online version within an adversarial bandit setting. However, their approach requisites the offline algorithm to possess an iterative greedy structure. In contrast, our framework treats the offline algorithm as a black-box algorithm. Furthermore, unlike our work, \citet{niazadeh2021online} imposes a condition known as Blackwell reproducibility on the offline algorithm, in addition to the resilience property. A closely related work is the single-agent framework by \citet{nie2023framework} called C-ETC, which adapts offline algorithms with robustness guarantees to stochastic combinatorial single-agent MAB, generalizing some previous works for submodular maximization \cite{nie2022explore, fourati2023randomized}. However, C-ETC fails to generalize the more recent work of \citet{fourati2024combinatorial}. Our framework generalizes and outperforms the results of \citet{nie2023framework} in many ways. First, it extends to the multi-agent scenario involving $m \geq 1$ communicating agents, including the single-agent scenario. Additionally, while the C-ETC framework cannot replicate the results of \citet{fourati2024combinatorial} for submodular maximization, ours not only recovers all these previous works, including the framework, but also achieves even tighter regret guarantees that decrease with an increasing number of selected agents.

\section{Combinatorial MA-MAB Framework}
\label{sec:alg}

We first define resilient approximation algorithms and then present our proposed offline-to-online framework.

\subsection{Offline Resilient-Approximation}
\label{sec:robust}  

We introduce the concept of resilient approximation, a metric that allows us to evaluate how an offline approximation algorithm reacts to controlled variations in function evaluations. 
We demonstrate that this specific characteristic alone can ensure that the offline algorithm can be modified to tackle stochastic C-MA-MAB settings, with only bandit feedback and achieving a sub-linear regret. Moreover, this adaptation does not rely on the algorithm's structure but treats it as a black-box algorithm. 

We define the $\xi$-controlled-estimation $\bar{f}$ of a reward function $f$ to deal with controlled variations.

\begin{definition}[$\xi$-controlled-estimation]\label{def:controlled_estimation}
For a set function $f:\mathbb{S}\to \mathbb{R}$ defined over a finite domain $\mathbb{S} \subseteq 2^\Omega$, a set function $\bar{f}:\mathbb{S}\to \mathbb{R}$, for a $\xi \geq 0$, is a $\xi$-controlled estimation of $f$ if $|f(S)-\bar{f}(S)|\leq \xi$ for all $S\in \mathbb{S}$.
\end{definition}

An estimation $\bar{f}$ is a $\xi$-controlled-estimation for a set function $f$ if it consistently stays within a small positive range $\xi$ of the actual values across all possible sets in $\mathbb{S}$. Given such a $\xi$-controlled-estimation for a set function $f$, we define $(\alpha, \beta, \gamma, \psi, \delta)$-resilient-approximation as follows:

\begin{definition}[$(\alpha, \beta, \gamma, \psi, \delta)$-resilient approximation]\label{def:resilient}
For any $\epsilon \geq 0$, an algorithm $\mathcal{A}(\epsilon)$ is an $(\alpha, \beta, \gamma, \psi, \delta)$-resilient approximation for maximizing a function $f: \mathbb{S} \subseteq 2^\Omega \to \mathbb{R}$ if, after making a number of calls $N(\beta, \gamma, \psi, \epsilon)$ to a $\xi$-controlled estimator $\bar{f}$, its output $\Theta$ satisfies the following condition:
$\mathbb{E}[f(\Theta)]\geq (\alpha - \epsilon) f(S^{\star}) -\delta \xi$, where $S^{\star}$ is the optimal set under $f$, and the expectation is over the randomness of $\mathcal{A}(\epsilon)$. Here, $N(\beta, \gamma, \psi, \epsilon)$ is defined as $\psi$ when $\epsilon=0$, and as $\psi \frac{1}{\epsilon^\beta} \log^{\gamma}(\frac{1}{\epsilon})$ otherwise.
\end{definition}

\begin{remark}
Several offline approximation algorithms are designed to solve specific combinatorial problems under particular reward and constraint assumptions. In \cref{appendix:application}, we study various ones and demonstrate their resilience, characterizing each by their corresponding parameters: $\alpha, \beta, \gamma, \psi,$ and $\delta$. Given the resilience of an offline algorithm, one can directly apply \cref{thm:main} to ascertain the corresponding worst-case online guarantees when extending that offline algorithm to online settings.
\end{remark}

\begin{remark}
An equivalent definition to resilient approximation was proposed by \citet{nie2023framework} as a robust approximation. However, the resilient approximation provides more information about the algorithm's complexity. Specifically, an $(\alpha-\epsilon, \delta)$-robust-approximation algorithm that requires a number of oracle calls equal to $\frac{\psi}{\epsilon^\beta} \log^\gamma(\frac{1}{\epsilon})$, then it is an $(\alpha, \beta, \gamma, \psi, \delta)$-resilient-approximation algorithm. Furthermore, an $(\alpha, \beta, \gamma, \psi, \delta)$-resilient-approximation algorithm is an $(\alpha-\epsilon, \delta)$-robust-approximation algorithm, with $\epsilon \geq 0$.
\end{remark}

Generally, the offline algorithms are designed for deterministic (noiseless) functions. However, in real-world applications, access to a noiseless reward function is not always possible, often due to the inherently stochastic nature of the problem. For example, recommending the same set of products to different people, or even to the same person at different times, may not yield consistent outcomes and rewards. This noise could also stem from using a stochastic approximation of the oracle function for computational efficiency, as seen in the case of stochastic data summarization discussed in \cref{data_summ_section}. Therefore, in such cases, resilience against noisy rewards is necessary. 

We demonstrate in \cref{thm:main} that resilience is, in fact, sufficient to ensure that the offline algorithm can be adjusted to achieve online sub-linear regret guarantees. In what follows, we explain how we use a resilient approximation algorithm $\mathcal{A}(\epsilon)$ as a subroutine in our proposed framework.

\subsection{Offline-to-Online Multi-Agent Framework}

\begin{algorithm}[t]
\caption{C-MA-MAB}
\label{alg:cmamab}
\begin{algorithmic}
    \STATE {\bfseries Input:}  Horizon $T$, actions $\mathbb{S}$, agents $\mathbb{A}'$, number $m$, $(\alpha, \beta, \gamma, \psi, \delta)$-resilient-approximation algorithm $\mathcal{A}(\epsilon)$
    \STATE  Initialize $r_i^\star \leftarrow \left\lceil m^{-1} \left(\delta \sqrt{\log(T)} \left(\frac{Tm}{\psi}\right)^\frac{1}{\beta+1} \right)^{\frac{2+2\beta}{3+\beta}} \right \rceil $
    \STATE  Initialize $\epsilon^\star \leftarrow (\frac{\psi r_i^\star}{T})^{\frac{1}{\beta+1}} \mathbf{1}_{\{\beta>0 \text{ OR } \gamma>0\}} $.
    \STATE 
    \STATE  \# Multi-Agent Exploration Time
    \STATE Server starts running $\mathcal{A}(\epsilon^\star)$, $j \leftarrow 0$
    \WHILE{$\mathcal{A}(\epsilon^\star)$ queries the value of some  $A\subseteq \mathbb{S}$}%
    \STATE Server broadcasts action $A$ to the agents, $j \leftarrow j + 1$
    \STATE Server randomly selects a set $\mathbb{A} \subseteq \mathbb{A}'$ of $m$ agents
        \FOR{agent $i$ in $\mathbb{A}'$ in parallel}
            \STATE For $r_i^\star$ times, play action $A$
            \STATE If $i \in \mathbb{A}$, agent tracks \& upload local mean $\bar{f}_i$ %
        \ENDFOR
        \STATE Server calculates the mean $\bar{f}$ and feeds it to $\mathcal{A}(\epsilon^\star)$
    \ENDWHILE
    \STATE %
    \STATE \# Multi-Agent Exploitation Time 
    \STATE Server broadcasts $\Theta$, the final output of $\mathcal{A}(\epsilon^\star)$
    \FOR{agent $i$ in $\mathbb{A}'$ in parallel}
        \FOR{\emph{remaining time of the agent $i$}}
            \STATE Play action $\Theta$
        \ENDFOR
    \ENDFOR
\end{algorithmic}
\end{algorithm}

We present our proposed C-MA-MAB Framework; see \cref{alg:cmamab}. This framework is applicable for both single-agent and multi-agent settings, wherein the design of our algorithm ensures that the single-agent setting is simply a special case of the multi-agent scenario. For a set $\mathbb{S}$ of possible actions, a set of $m' \geq 1$ agents $\mathbb{A}'$, a number $m \geq 1$ of communicating agents, and a time horizon $T$, our algorithm can adapt any off-the-shelf, offline single-agent combinatorial $(\alpha, \beta, \gamma, \psi, \delta)$-resilient-approximation algorithm $\mathcal{A}(\epsilon)$ to the C-MA-MAB setting. This adaptation comes with theoretical guarantees, as detailed in \cref{thm:main}, applicable under any reward type or action constraint.

While the offline algorithm may be a function of some variable $\epsilon \geq 0$, which trades off its complexity and approximation guarantees, our proposed algorithm finds $\epsilon^\star$ which optimizes this trade-off based on the problem parameters and its complexity to minimize regret and uses $\mathcal{A}(\epsilon^\star)$ for exploration instead. Furthermore, in the exploration time, whenever the offline algorithm $\mathcal{A}(\epsilon^\star)$ requests the value oracle for action $A$, each agent of the $m'\geq 1$ agents plays the action $A$ for $r_i^\star =  r^\star/m$ times, where $r^\star$ is another parameter chosen based on the setting and the subroutine complexity to minimize regret, with
\begin{equation}
\label{r^*}
    r^{\star}=m\left\lceil m^{-1}\left(\delta \sqrt{\log(T)} \left(\frac{Tm}{\psi}\right)^\frac{1}{\beta+1} \right)^{\frac{2+2\beta}{3+\beta}}\right\rceil.
\end{equation}
Furthermore, we choose $\epsilon^\star$ as follows:
\begin{equation}
\label{epsilon_star}
    \epsilon^\star = (\frac{\psi r^\star}{T m})^{\frac{1}{\beta+1}} \mathbf{1}_{\{\beta>0 \text{ OR } \gamma>0\}}.
\end{equation}

Only the randomly selected $m$ agents track and broadcast their local estimations to the server, i.e., each agent $i \in \mathbb{A}$ sends its local estimation $\bar{f}_i$, then the server aggregates these estimations in one global estimation $\bar{f}$ of rewards for $A$ and then returns $\bar{f}$ to $\mathcal{A}(\epsilon^\star)$. Finally, in the exploitation phase, all the agents in $\mathbb{A}'$ play $\Theta$, the output from algorithm $\mathcal{A}(\epsilon^\star)$, for the remaining time. 

\begin{remark}
C-MA-MAB has low storage complexity. In every step, an agent needs to store, at most, $n$ indices and a real value representing the empirical mean of one action. Only the action $\Theta$ is stored during exploitation time, and no additional computation is required. During exploration time, each agent needs to store only the proposed action $A$ and update its associated empirical mean; everything is deleted once the server proposes another action. Furthermore, the proposed framework does not require the combinatorial offline algorithm $\mathcal{A}(\epsilon)$ to have any particular structure and employs $\mathcal{A}(\epsilon^\star)$ as a black-box algorithm. Consequently, it shares the same complexity as the subroutine $\mathcal{A}(\epsilon^\star)$. Moreover, $\mathcal{A}(\epsilon^\star)$ is executed on the server, alleviating the computational overhead for the agents.
\end{remark}

\begin{remark}
In the multi-agent setting, we assume that the server can either be a separate entity or one of the agents playing the role of the server. In the single-agent setting, without loss of generality, the sole agent can be considered as the server. The server orchestrates communication by selecting clients and recommending which actions to explore, based on the offline subroutine, in a synchronized manner. Recent studies have explored federated linear and kernelized contextual bandits with asynchronous communication \cite{li2022asynchronous, he2022simple, li2023learning}. Future research might investigate general combinatorial optimization with asynchronous communication.
\end{remark}

\begin{remark}
The proposed C-MA-MAB uses the time horizon $T$ to compute $r^\star$ and $\epsilon^\star$. When the exact time horizon is unknown, the results can be enhanced by employing the concept of an anytime algorithm through the use of the geometric doubling trick by establishing a geometric sequence of time intervals, denoted as $T_i$, where $T_i=T_0 2^i$ for $i \in \mathbb{N}$, where $T_0$ is a sufficiently large value to ensure proper initialization. From Theorem 4 in \cite{besson2018doubling}, it follows that the regret bound preserves the $T^{\frac{2+\beta}{3+\beta}}$ dependence with only changes in constant factors. 
\end{remark}

\section{Theoretical Analysis}

We upper-bound the required communication rounds, we lower-bound the probability of the empirical mean being $\xi$-controlled-estimation of the expectation (\textit{good event}), and upper-bound the expected cumulative $\alpha$-regret.

\subsection{Communication Analysis}

By the design of the algorithm, the $m$ randomly selected agents, after locally estimating the quality of a suggested action $A$, communicate only one value, representing the local estimation $\bar{f}_i$. The agents exploit the decided set $\Theta$ during the exploitation phase without further communication. During exploration, the selected agents must communicate their local estimation for every requested action $A$. Therefore, the number of communication rounds is upper-bounded by the number of requested actions, i.e., the required oracle calls $N(\beta, \gamma, \psi, \epsilon^\star),$ which we upper-bound in the following lemma, which we prove in Appendix \ref{sec:appd:proof:communication}.

\label{sec:appd:proof:queries}
\begin{lemma} \label{lem:queries}
The number of communication times, i.e., the number of oracle queries $N(\beta, \gamma, \psi, \epsilon^\star)$ of the subroutine $(\alpha, \beta, \gamma, \psi, \delta)$-resilient-approximation algorithm $\mathcal{A}(\epsilon^\star)$ satisfies:
$N(\beta, \gamma, \psi, \epsilon^\star) %
 \leq \mathcal{O}(\psi T^\frac{\beta}{\beta+1}\log^{\gamma}(T)).$
\end{lemma}

From \cref{lem:queries} it follows that for cases where $\beta=0$, which applies to several offline algorithms \cite{nemhauser1978analysis, khuller1999budgeted, sviridenko2004note, buchbinder2015tight, mirzasoleiman2015lazier, yaroslavtsev2020bring}, the communication rounds are at most $\widetilde{\mathcal{O}}(\psi)$, scaling at most logarithmically with $T$. For example, as shown in \cref{cor:RandomUSM:robust}, with $n$ arms and using RandomizedUSM \cite{buchbinder2015tight} as a subroutine (where $\psi$ is $n$, $\beta$ is zero, and $\gamma$ is zero), our framework guarantees a communication complexity of $\mathcal{O}(n)$, not scaling with $T$.

By design of the C-MA-MAB algorithm, after every action queried by the subroutine, the agents have to explore and estimate the values of the proposed action. To do that, each agent has to play the proposed action for $r_i^\star$ times. Therefore, using the result from \cref{lem:queries} on the number of required communication rounds and by the definition of $r_i^\star$, we can derive that the required exploration steps for every agent, i.e., $\mathcal{O}(r_i^\star \psi T^\frac{\beta}{\beta+1}\log^{\gamma}(T))$, which is decreasing with an increasing number of agents $m$.

\subsection{Estimation Analysis}

In C-MA-MAB, every agent plays each action queried by the subroutine $\mathcal{A}(\epsilon^\star)$ the same number of times. These repetitions provide an estimation of the action values. We define a \textit{good event} $\mathcal{E}$ when the empirical mean estimation $\bar{f}$ is a $\xi$-controlled-estimation of the reward expectation $f$, on the played actions during exploration time, with $\xi := \sqrt{\log(T)/r^\star}$. For every communication round $j$, each action $A_j$ queried by the $(\alpha, \beta, \gamma, \psi, \delta)$-resilient-approximation $\mathcal{A}(\epsilon^\star)$, where $j \in \{1, \cdots, N(\beta, \gamma, \psi, \epsilon^\star)\}$, we define the event $\mathcal{E}_{j}$ as:
\begin{align}
    \mathcal{E}_{j} \triangleq \{\big|\bar{f}(A_j)-f(A_j) \big|\leq \xi\}. \label{eq:iteration_event}
\end{align}

Therefore, the \textit{good event} $\mathcal{E}$, which considers the empirical mean estimation $\bar{f}$ is a $\xi$-controlled estimate of the reward expectation $f$ for every communication round $j$, i.e., considers the realization of $\mathcal{E}_{j}$ for every $j \in \{1, \cdots, N(\beta, \gamma, \psi, \epsilon^\star)\}$, which is expressed as follows:
\begin{align}
    \mathcal{E} = \mathcal{E}_{1} \cap \dots \cap \mathcal{E}_{N(\beta, \gamma, \psi, \epsilon^\star)}. \label{eq:clean_event}
\end{align}

Each action $A$ queried by the offline algorithm have been explored for $r^\star$ number of times among the agents. These $r^\star$ rewards are i.i.d. with expectation $f(A)$ and confined within the $[0,1]$ range. Consequently, we can bound the deviation of the empirical mean $\bar{f}(A_j)$ from the expected value $f(A_j)$ for every action undertaken.
Thus, we upper bound the probability of the \textit{good event} in the following lemma, which we prove in \cref{sec:appd:proof:clean-event}.
\begin{lemma} \label{lem:probcleanevents}
The probability of the \textit{good event} $\mathcal{E}$, \eqref{eq:clean_event}, when using an $(\alpha, \beta, \gamma, \psi, \delta)$-resilient-approximation algorithm $\mathcal{A}(\epsilon^\star)$ as a subroutine satisfies:
\begin{align}
    \mathbb{P}(\mathcal{E}) 
    & \geq 1 - 2N(\beta, \gamma, \psi, \epsilon^\star)T^{-2} . \nonumber
\end{align}
\end{lemma}
Combining both results of \cref{lem:queries} and \cref{lem:probcleanevents} it follows that the \textit{bad event} happens with a probability of at most $\Tilde{\mathcal{O}}(\psi T^{\frac{\beta}{\beta+1}-2})$, decreasing as $T$ increases.

\subsection{Regret Analysis}
We analyze the expected cumulative $\alpha$-regret for the C-MA-MAB (\cref{alg:cmamab}), with $m$ communicating agents. %

\begin{theorem} \label{thm:main}
For the sequential combinatorial decision-making problem defined in Section \ref{prob_state}, with $T\geq \max\{\psi m,\psi m^{\frac{1+\beta}{2}}/\delta^{\beta+1}\}$, the expected cumulative $\alpha$-regret of the C-MA-MAB presented in Algorithm \ref{alg:cmamab} using an $(\alpha, \beta, \gamma, \psi, \delta)$-resilient-approximation algorithm $\mathcal{A}(\epsilon^\star)$ as subroutine is at most $\Tilde{\mathcal{O}}\left(m^{-\frac{1}{3+\beta}} \delta^\frac{2}{3+\beta}\psi^\frac{1}{3+\beta}  T^\frac{2+\beta}{3+\beta}\right)$.
\end{theorem} 

The above theorem implies that an offline algorithm does not need to have an $\alpha$-approximation guarantee to be adapted to achieve sublinear $\alpha$-regret guarantees. In fact, an $(\alpha-\epsilon)$-approximation algorithm, denoted as $\mathcal{A}(\epsilon)$ with $\epsilon > 0$, if it is a $(\alpha, \beta, \gamma, \psi, \delta)$-resilient-approximation, can be extended to a sub-linear $\alpha$-regret algorithm. Later, in Section \ref{sec:appl-CETC:submod}, we apply the above theorem to special combinatorial cases.

\begin{remark}
 Linear speedup is evident in our approach, as the collective regret across $m$ agents is  $\Tilde{\mathcal{O}}\left(\delta^\frac{2}{3+\beta}\psi^\frac{1}{3+\beta}  (Tm)^\frac{2+\beta}{3+\beta}\right)$. This mirrors the regret one would observe if all $m$ agents collaborated, sharing a total of $Tm$ time for the central agent. Consequently, the distributed setup incurs no loss, with each agent interacting with the environment $T$ times, and the combined regret reflects that of a single agent allocated $Tm$ time for interaction. 
\end{remark}

\begin{remark}
The $\delta$ function depends on the offline algorithm and refers to a general function of the combinatorial problem parameters, such as the number of main arms, $n$, or the cardinality constraint, $k$. This function serves as the scaling factor for the radius $\xi$ in the lower bound on the expected reward, as given by $(\mathbb{E}[f(\Theta)] \geq (\alpha - \epsilon) f(S^{\star}) - \delta \xi)$, as defined in \cref{def:resilient} for $(\alpha, \beta, \gamma, \psi, \delta)$-resilient-approximation algorithms. In some offline algorithms, we have demonstrated that this scaling function depends on the cardinality constraint $k$. For example, Lemma E.4 shows that $\delta$ equals $4k$ for \textsc{RandomSampling}. In other cases, the function may depend solely on the number of main arms $n$. For instance, Lemma E.1 establishes that $\delta$ is $\frac{5}{2}n$ for the \textsc{RandomizedUSM}. Our analysis accommodates any $\delta$ function defined in terms of the problem parameters.
\end{remark}

\begin{remark}
When algorithm approximations do not depend on $\epsilon$, it implies that their complexity is of the form $\tilde{\mathcal{O}}(\psi)$, hence $\beta=\gamma=0$. It follows that using such an $(\alpha, 0, 0, \psi, \delta)$-resilient-approximation algorithm as a subroutine achieves a regret of at most $\Tilde{\mathcal{O}}\left(m^{-\frac{1}{3}} \delta^\frac{2}{3}\psi^\frac{1}{3} T^\frac{2}{3}\right)$.
\end{remark}

\begin{remark}
A lower bound remains an open question for general combinatorial stochastic rewards under bandit feedback. A lower bound is missing even for the special cases of stochastic submodular rewards under bandit feedback. Some lower bounds have been proposed in restrictive special settings. For example, \citet{niazadeh2021online} showed a $\tilde{\Omega}(T^\frac{2}{3})$ lower bound for adversarial submodular rewards, where the reward could only be observed in user-specified exploration rounds. Moreover, \citet{tajdini2023minimax} demonstrated that for monotone stochastic submodular bandits with a cardinality constraint, for small time horizon $T$, a regret scaling like $T^\frac{2}{3}$ is inevitable when compared to the greedy algorithm in \cite{nemhauser1978analysis}. However, it does not provide a lower bound on $(1-1/e)$-regret. 
\end{remark}

In the following, we provide a sketch of the proof and leave a detailed one in Appendix \ref{appendix:proof:regret}. We separate the proof into two cases. One case when the \textit{good event} $\mathcal{E}$  
happens, which we show in \cref{lem:probcleanevents} happens with high probability and then we generalize the result under any event. 

\subsubsection{regret of an Agent under the Good Event}

We upper-bound the expected $\alpha$-regret conditioned on the \textit{good event} $\mathcal{E}$. However, for simplicity in notation, we employ $\mathbb{E}[\cdot]$ rather than $\mathbb{E}[\cdot|\mathcal{E}]$ in certain instances. We decompose and then bound the expected $\alpha$-regret into two components: one addressing the regret stemming from exploration ($P_1$) and the other from exploitation ($P_2$). 
\begin{align}
    \mathbb{E}[\mathcal{R}_i(T)|\mathcal{E}] 
    & = \sum_{t=1}^T \left(\alpha f( \mathrm{S^\star} ) - \mathbb{E}[f(S_{i,t})]\right)  \\
    & = \underbrace{\sum_{j=1}^{N(\beta, \gamma, \psi, \epsilon^\star)} r_i^\star \left(\alpha f(\mathrm{S^\star})-\mathbb{E}[f(A_j)] \right)}_{\text{$\alpha$-regret from exploration ($P_1$)}} \nonumber \\
    &+ \underbrace{\sum_{t=T_{N(\beta, \gamma, \psi, \epsilon^\star)}+1}^T \left(\alpha f(\mathrm{S^\star})-\mathbb{E}[f(\Theta)] \right)}_{\text{$\alpha$-regret from exploitation ($P_2$)}} \label{eq:prf:main:case1:60} 
\end{align}

We begin with bounding regret from exploration and using that the rewards are within the interval $[0,1]$,
\begin{align}
    P_1 &\leq  \sum_{j=1}^{N(\beta, \gamma, \psi, \epsilon^\star)} \frac{r^\star}{m} \alpha \leq N(\beta, \gamma, \psi, \epsilon^\star)\frac{r^\star}{m}. \label{eq:main:cum-regr-explor}
\end{align}
When the \textit{good event} $\mathcal{E}$ occurs, we know that $|\bar{f}(A)-f(A)|\leq \xi$ for all considered action $A$. Using an $(\alpha, \beta, \gamma, \psi, \delta)$-resilient-approximation $\mathcal{A}(\epsilon^\star)$, with output $\Theta$, we have
\begin{align}
\label{eq:final_exp_reward}
\alpha f(S^\star)-\mathbb{E}[f(\Theta)]\leq \delta \xi + \epsilon^\star f(S^\star). 
\end{align}
Therefore, with $f(\mathrm{S^\star})<1$, we have:
\begin{align}
    P_2 
     & \leq \sum_{t=T_{N(\beta, \gamma, \psi, \epsilon^\star)}+1}^T (\delta \xi + \epsilon^\star)  \leq T(\delta \xi+\epsilon^\star). \label{eq:main:cum-regr-exploit}
\end{align}
Therefore, using Eq. \eqref{eq:main:cum-regr-explor} and Eq. \eqref{eq:main:cum-regr-exploit}, the total expected cumulative regret in Eq. \eqref{eq:prf:main:case1:60} can be bounded as:
\begin{align}
    \mathbb{E}[\mathcal{R}_i(T)|\mathcal{E}] 
    &\leq N(\beta, \gamma, \psi, \epsilon^\star)\frac{r^\star}{m} + T(\delta \xi+\epsilon^\star).
\end{align}
Using the confidence radius $\xi = \sqrt{\log(T)/r^\star}$ and $N(\beta, \gamma, \psi, \epsilon^\star)=\psi \frac{1}{{\epsilon^\star}^\beta} \log^{\gamma}(\frac{1}{\epsilon^\star})$, we have 
\begin{align}
    \mathbb{E}[\mathcal{R}_i(T)|\mathcal{E}] 
    & \leq  \frac{\psi r^\star \log^{\gamma}(\frac{1}{\epsilon^\star}) }{{\epsilon^\star}^\beta m}   +  T( \sqrt{\frac{\delta^2\log(T)}{r^\star}} + \epsilon^\star). \nonumber 
\end{align}
We note that the above inequality is correct for all values of $r^\star \geq m$ and $\epsilon^\star \geq 0$ with the convention that $0^0=1$. In our algorithm, we choose the values of $r^\star$ and $\epsilon^\star$ as functions of the problem parameters and on the subroutine complexity parameters. We choose $r^\star$ as defined in Eq. \eqref{r^*} and $\epsilon^\star$ as defined in Eq. \eqref{epsilon_star}. 

Recall that $\beta \geq 0$ and $\gamma \in \{0,1\}$. Therefore, we consider all the possible cases, the first when $\beta = \gamma = 0$, the second when $\beta =0$ and $\gamma = 1$, and the third when $\beta > 0$. For all the cases, when the \textit{good event} $\mathcal{E}$ happens, the expected $\alpha$-regret of our C-MA-MAB with an $(\alpha, \beta, \gamma, \psi, \delta)$-resilient-approximation as subroutine we have
\begin{align}
\mathbb{E}[\mathcal{R}_i(T)|\mathcal{E}] &\leq \Tilde{\mathcal{O}}\left(\delta^\frac{2}{3+\beta}\psi^\frac{1}{3+\beta} m^{-\frac{1}{3+\beta}} T^\frac{2+\beta}{3+\beta} \right). \label{main:final_upperbound_clean}
\end{align}

\subsubsection{Regret of an Agent under Any Event}

Given that the reward $f_t(\cdot)$ is upper bound by 1, the expected cumulative $\alpha$-regret when the event $\bar{\mathcal{E}}$ happens over a range $T$ is upper-bounded as follows:
$
    \mathbb{E}[\mathcal{R}_i(T)|\bar{\mathcal{E}}] \leq T. 
$
Combining the results when \textit{good event} happens and does not happen, using the law of total expectation, and using Eq. \eqref{main:final_upperbound_clean}, \cref{lem:queries}, \cref{lem:probcleanevents}, and $T \geq m\psi$:
\begin{align}
    \mathbb{E}[\mathcal{R}_i(T)] &= \mathbb{E}[\mathcal{R}_i(T)|\mathcal{E}] \cdot \mathbb{P}(\mathcal{E}) +\mathbb{E}[\mathcal{R}_i(T)|\bar{\mathcal{E}}]\cdot \mathbb{P}(\bar{\mathcal{E}}) \nonumber \\
    &\leq \Tilde{\mathcal{O}}\left(\delta^\frac{2}{3+\beta}\psi^\frac{1}{3+\beta} m^{-\frac{1}{3+\beta}} T^\frac{2+\beta}{3+\beta} \right) \nonumber   .
\end{align}
This establishes the result in \cref{thm:main}.

\section{Application to Submodular Maximization} \label{sec:appl-CETC:submod}

We use our C-MA-MAB framework to address scenarios involving stochastic submodular\footnote{A function $f: 2^{\Omega} \rightarrow \mathbb{R}$, over a finite set $\Omega$, is considered submodular if it discloses the characteristic of diminishing returns: for any $A \subseteq B \subset \Omega$ and $v \in \Omega \backslash B$, the inequality $f(A \cup \{v\}) - f(A) \geq f(B \cup \{v\}) - f(B)$ is verified.} rewards, with bandit feedback. Submodular maximization (SM) is an NP-hard problem \cite{nemhauser1978analysis, feige2011maximizing}, which recently has shown growing interest in studying combinatorial MAB \citep{chen2018contextual, niazadeh2021online, nie2022explore, fourati2023randomized, fourati2024combinatorial}. In the following, we present our results for SM for monotone\footnote{A function $f: 2^{\Omega} \rightarrow \mathbb{R}$, over a finite set $\Omega$, is considered monotone if for any $A \subseteq B \subseteq \Omega$ we have $f(A) \leq f(B)$.} and non-monotone rewards with and without a cardinality constraint and leave the knapsack constraint in Appendix \ref{knapsack_appendix}.

For unconstrained SM (USM), \citet{buchbinder2015tight} proposed \textsc{RandomizedUSM}, achieving a $\frac{1}{2}$-approximation. In \cref{lem:RandomUSM:robust} in Appendix \ref{appendix:application}, we generalize Corollary 2 in \cite{fourati2023randomized} to show its resilience and present the following corollary, which recovers the guarantees of the online algorithm in \cite{fourati2023randomized} for a single agent and generalizes it for multi-agent setting.

\begin{corollary}\label{cor:RandomUSM:robust}
    C-MA-MAB, using the \textsc{RandomizedUSM} as a subroutine, needs at most $\mathcal{O}(n)$ communication times and its $\frac{1}{2}$-regret is at most $\tilde{\mathcal{O}}\left(m^{-\frac{1}{3}}nT^\frac{2}{3}\right)$ for USM.
\end{corollary}

For SM under a cardinality constraint $k$ (SMC) with monotone rewards, the \textsc{Greedy} in \citep{nemhauser1978best} achieves $1-1/e$. In contrast, the \textsc{Stochastic-Greedy} in \citep{mirzasoleiman2015lazier} achieves $1-1/e-\epsilon$, where $\epsilon$ is a parameter balancing accuracy and complexity. We provide the resilience of these two algorithms in Appendix \ref{appendix:application} and present the following results.

\begin{corollary}
\label{cor:greedy}
    C-MA-MAB, using the \textsc{Greedy} as a subroutine, needs $\mathcal{O}(nk)$ communication times and its $(1-1/e)$-regret is at most $\tilde{\mathcal{O}}\left(m^{-\frac{1}{3}} k n^\frac{1}{3} T^\frac{2}{3}\right)$ for monotone SMC.
\end{corollary}

\begin{corollary}
\label{cor:stochasticgreedy}
    C-MA-MAB, using the \textsc{Stochastic-Greedy}, needs $\tilde{\mathcal{O}}(n)$ communication times and its $(1-\frac{1}{e})$-regret is at most $\tilde{\mathcal{O}}\left(m^{-\frac{1}{3}} k^\frac{2}{3} n^\frac{1}{3} T^\frac{2}{3}\right)$ for monotone SMC.
\end{corollary}

The $\textsc{Stochastic-Greedy}$ algorithm has sub-optimal approximation guarantees of $(1-1/e-\epsilon)$; thus, using the C-ETC framework from \cite{nie2023framework} can only guarantee sub-linear $(1-1/e-\epsilon)$-regret. Consequently, $(1-1/e)$-regret will be linear in $T$. However, our C-MA-MAB guarantees sublinear $(1-1/e)$-regret, recovering the single agent's results in \cite{fourati2024combinatorial}. Furthermore, the two corollaries above demonstrate that employing a sub-optimal approximation algorithm in terms of the approximation factor, rather than one that achieves optimal approximation guarantees, does not necessarily imply lower regret guarantees, as shown in \cite{fourati2024combinatorial}. 

For non-monotone SMC, the \textsc{RandomSampling} algorithm in \citep{buchbinder2017comparing} achieves $1/e-\epsilon$, where $\epsilon$ is a parameter balancing accuracy and complexity. We provide the resilience of this algorithm in Appendix \ref{appendix:application} and derive the first result for single-agent and multi-agent online stochastic non-monotone SMC with sublinear regrets.

\begin{corollary}
\label{cor:randomsampling}
    C-MA-MAB, using the \textsc{Random Sampling} in \citep{buchbinder2017comparing} as a subroutine, needs $\tilde{\mathcal{O}}(n T^\frac{2}{3})$ communication times and its $(\frac{1}{e})$-regret is at most $\tilde{\mathcal{O}}\left(m^{-\frac{1}{5}} k^\frac{2}{5} n^\frac{1}{5}  T^\frac{4}{5}\right)$ for non-monotone SMC.
\end{corollary}

\begin{figure}[t]
    \centering
    \includegraphics[width=0.39\textwidth]{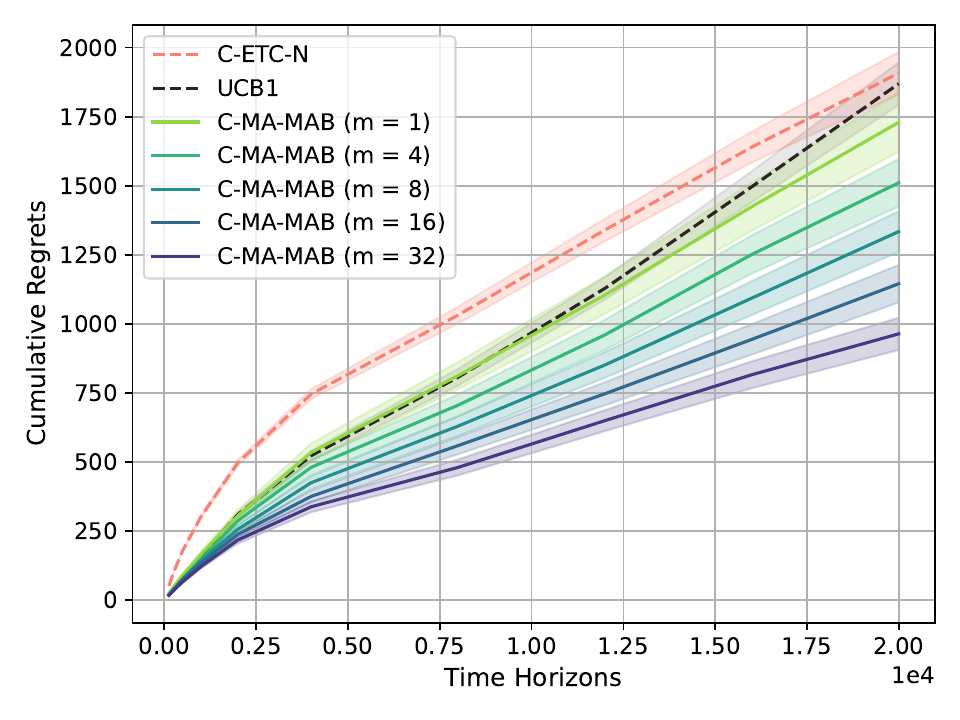}
    \caption{\small Cumulative regrets of summarizing images from CIFAR10 for different horizons $T$ using our C-MA-MAB framework with different number of agents $m$, against C-ETC-N and UCB1.}
    \label{fig:regrets_cifar10}
\end{figure}

\section{Experiments with Data Summarization}
\label{data_summ_section}

We employ our C-MA-MAB on data summarization, a primary challenge in machine learning \citep{mirzasoleiman2013distributed}, mainly when dealing with a large dataset. While this problem has been widely studied with access to a deterministic oracle \cite{lin2011class, mirzasoleiman2013distributed, mirzasoleiman2015lazier, mirzasoleiman2020coresets, sivasubramanian2024gradient}, this work is the first to address online data summarization under a stochastic objective function. We run experiments on FMNIST \cite{xiao2017fashion} and CIFAR10 \cite{krizhevsky2009learning}, present the latter in the main paper, and relegate more details and results to Appendix \ref{additional_experiements}. 

In data summarization, an action $A$ consists of a set of at most $k$ images to summarize a large dataset $\mathcal{D}$.
Adding more images achieves better summarization but follows a diminishing return property. Thus, it falls in the monotone SMC \cite{mirzasoleiman2015lazier}. Evaluating a given action $A$ against a dataset $\mathcal{D}$ may become expensive with a large dataset. Thus, we consider a stochastic objective where the chosen subset is compared only to a random subset $\mathcal{R} \subseteq \mathcal{D}$ drawn uniformly at random from $\mathcal{D}$, using a similarity metric $C$, resulting in noisier but lower complexity evaluations. We do not solve the problem for a given realization $\mathcal{R}$, but we solve it in expectation:
$
\underset{\mathcal{A} \subseteq \mathcal{D}:|\mathcal{A}| \leq k}{\arg \max } \mathbb{E}_{\mathcal{R}} \left[ \sum_{i \in \mathcal{R}} \max _{v \in \mathcal{A}} C(i,v)   \right]. 
$

We test our method when using the $\textsc{Stochastic-Greedy}$ algorithm as a subroutine \cite{mirzasoleiman2015lazier} for one agent and multiple agents and compare it to the proposed algorithm in C-ETC framework for SMC (C-ETC-N) \cite{nie2023framework}, and the upper confidence bound (UCB1) algorithm \cite{auer2002finite}. 

The C-MA-MAB demonstrates sub-linear regret guarantees as depicted in Fig. \ref{fig:regrets_cifar10}. Additionally, it is apparent that, for varying values of $m$, the C-MA-MAB consistently outperforms both C-ETC-N and UCB1, even with a single agent, exhibiting lower regrets over diverse time horizons. Notably, an increase in the number of agents correlates with a reduction in regret for these agents. These observations reinforce the same conclusions drawn from the theoretical analysis.

\section*{Conclusion}
We introduce C-MA-MAB, a framework for single-agent and multi-agent online stochastic combinatorial problems, which adapts resilient offline $(\alpha-\epsilon)$-approximation algorithms to online algorithms under bandit feedback, achieving sublinear $\alpha$-regret bounds with respect to the time horizon $T$, eliminating the $\epsilon$ error, and ensuring a linear speedup. We also present specialized bounds for SM with and without constraints and apply C-MA-MAB to online stochastic data summarization.

\section*{Impact Statement}
This paper presents work whose goal is to advance the field of Machine Learning. There are many potential societal consequences of our work, none which we feel must be specifically highlighted here.

\bibliographystyle{icml2024}
\bibliography{main}

\onecolumn
\appendix

\section{Notation}

In the following, for a given agent $i$ when we discuss any feasible action denoted as $A$, we use $f_t(A)$ to represent the realization of the stochastic reward at time $t$ when taking that action. We denote the expectation of the reward of playing that action as $f(A)$. We also introduce $\bar{f}_t(A)$, which is the empirical mean of rewards received from playing action $A$ up to and including time $t$. 

We omit the subscript $t$ when we write $\bar{f}(A)$, assuming it is clear that action $A$ has been played $r^\star$ times. We use $A_j$, with $j$ ranging from 1 to the total number of the approximation algorithm queries $N(\beta, \gamma, \psi, \epsilon^\star)$, to refer to the j-th action the algorithm queries. Additionally, we define $T_j$, where $j$ also varies from 1 to the total number of the offline algorithm queries $N(\beta, \gamma, \psi, \epsilon^\star)$, as the time step when $A_j$ has been played $r^\star$ times.

\section{Communication Rounds Analysis}
\label{sec:appd:proof:communication}

In this subsection we prove Lemma \ref{lem:queries} which upperbounds the number of oracle queries $N(\beta, \gamma, \psi, \epsilon^\star)$ of the offline algorithm $\mathcal{A}(\epsilon^\star)$ as follows:
\begin{align}
    N(\beta, \gamma, \psi, \epsilon^\star) %
    & \leq \mathcal{O}(\psi T^\frac{\beta}{\beta+1} \log^{\gamma}(T)) \nonumber
\end{align}
We provide examples of the resulting number of oracle queries for different offline algorithms or class transformations in \cref{tab:communications}, as shown in \cref{lem:queries}.
\begin{proof}
Using an $(\alpha, \beta, \gamma, \psi, \delta)$-robust approximation as subroutine, we have after $N(\beta, \gamma, \psi, \epsilon) = \psi \frac{1}{\epsilon^\beta} \log^{\gamma}(\frac{1}{\epsilon})$ oracle calls, for $\epsilon \geq 0$, $\beta \geq 0$, and $\gamma \in \{0,1\}$.

Recall that $\epsilon^\star$ is as follows:
\begin{equation}
    \epsilon^\star = (\frac{\psi r^\star}{T m})^{\frac{1}{\beta+1}} \mathbf{1}_{\{\beta>0 \text{ OR } \gamma>0\}}.
\end{equation}


Recall that $\beta \geq 0$ and $\gamma \in \{0,1\}$. Therefore, we consider all the possible cases (three), the first when $\beta = \gamma = 0$, the second when $\beta =0$ and $\gamma > 0$, and the third when $\beta > 0$.

\textbf{Case 1: $\beta = \gamma = 0$.}
Using Eq.\eqref{epsilon_star}, $\epsilon^\star$ becomes the following
\begin{equation}
    \epsilon^\star = 0.
\end{equation}

\begin{align}
     N(\beta, \gamma, \psi, \epsilon^\star)
    &=2 \psi \nonumber \\
    &= \mathcal{O}(\psi T^\frac{\beta}{\beta+1} \log^{\gamma}(T))
\end{align}

\textbf{Case 2: $\beta = 0$ and $\gamma \neq 0$.} 
We have $\gamma \neq 0$ and $\gamma \in \{0,1\}$, thus $\gamma = 1$.

Using Eq.\eqref{epsilon_star}, $\epsilon^\star$ becomes the following
\begin{equation}
    \epsilon^\star = \frac{\psi r^\star}{T m} .
\end{equation}

\begin{align}
     N(\beta, \gamma, \psi, \epsilon^\star)
    &=2 \psi (\frac{1}{\epsilon^\star})^\beta \log^{\gamma}(\frac{1}{\epsilon^\star}) \nonumber\\
    &=2 \psi \log(\frac{1}{\epsilon^\star}) \nonumber\\
    &=2 \psi \log(\frac{T m}{\psi r^\star}) \nonumber\\
    &=2 \psi \log(\frac{T}{\psi r_i^\star}) \nonumber\\
    &\leq 2 \psi \log(T) \nonumber\\
    &= \mathcal{O}(\psi \log(T)) \nonumber \\
    &= \mathcal{O}(\psi T^\frac{\beta}{\beta+1} \log^{\gamma}(T))
\end{align}

\textbf{Case 3: $\beta > 0$.}
Using Eq.\eqref{epsilon_star}, $\epsilon^\star$ becomes the following
\begin{equation}
    \epsilon^\star = (\frac{\psi r^\star}{T m})^{\frac{1}{\beta+1}} .
\end{equation}
Therefore,
\begin{align}
     N(\beta, \gamma, \psi, \epsilon^\star)
    &=2 \psi (\frac{1}{\epsilon^\star})^\beta \log^{\gamma}(\frac{1}{\epsilon^\star}) \nonumber\\
    &=2 \psi (\frac{1}{\beta + 1})^\gamma (\frac{T m}{\psi r^\star})^\frac{\beta}{\beta+1} \log^{\gamma}(\frac{T m}{\psi r^\star}) \nonumber\\
    &=2 \psi (\frac{1}{\beta + 1})^\gamma (\frac{T}{\psi r_i^\star})^\frac{\beta}{\beta+1} \log^{\gamma}(\frac{T}{\psi r_i^\star}) \nonumber\\
    &\leq 2 \psi (\frac{1}{\beta + 1})^\gamma (T)^\frac{\beta}{\beta+1} \log^{\gamma}(T) \nonumber\\
    &= \mathcal{O}(\psi T^\frac{\beta}{\beta+1} \log^{\gamma}(T))
\end{align}
\end{proof}

\begin{table*}[t]
\small
\centering
\begin{tabular}{|c|c|c|}
\hline 
\begin{tabular}[c]{@{}c@{}} Offline Algorithm  \end{tabular}   &  \begin{tabular}[c]{@{}c@{}} Offline \\ Complexity \end{tabular}   &   \begin{tabular}[c]{@{}c@{}} C-MA-MAB Resulted \\ Communication Complexity \end{tabular}                                                         \\ \hline 
\begin{tabular}[c]{@{}c@{}} \textsc{RandomizedUSM} \\ \cite{buchbinder2015tight} \end{tabular} & $\mathcal{O}\left(n\right)$  & $\mathcal{O}\left(n\right)$                                  
                                              \\ \hline
\begin{tabular}[c]{@{}c@{}} \textsc{Greedy} \\ \cite{nemhauser1978analysis}   \end{tabular} & $\mathcal{O}\left(nk\right)$ & $\mathcal{O}\left(nk\right)$
\\ \hline
\begin{tabular}[c]{@{}c@{}} \textbf{\textsc{General}} \\ \small{Excluding the next rows}\end{tabular}  & $\mathcal{O}(\psi)$  & $\mathcal{O}(\psi)$ 
\\ \hline

\begin{tabular}[c]{@{}c@{}} \textsc{Stochastic-Greedy} \\ \small{\cite{mirzasoleiman2015lazier}}   \end{tabular} & $\mathcal{O}\left(n \log(\frac{1}{\epsilon})\right)$ & $\mathcal{O}(n \log(T))$  
\\ \hline

\begin{tabular}[c]{@{}c@{}} \textbf{\textsc{More General}} \\ \small{Excluding the next rows} \end{tabular} & $\mathcal{O}(\psi \log^{\gamma}(\frac{1}{\epsilon}))$     &       $\mathcal{O}(\psi \log^{\gamma}(T))$          
\\ \hline
\begin{tabular}[c]{@{}c@{}} \textsc{RandomSampling} \\ \cite{buchbinder2017comparing}   \end{tabular} & $\mathcal{O}\left(\frac{n}{\epsilon^2} \log(\frac{1}{\epsilon})\right)$ &  $\mathcal{O}(\psi T^\frac{2}{3} \log(T))$                                                                            
\\ \hline
\begin{tabular}[c]{@{}c@{}} \textbf{\textsc{Most General}} \\ \small{Including the previous rows} 
 \end{tabular} & $\mathcal{O}( \frac{\psi}{\epsilon^\beta} \log^{\gamma}(\frac{1}{\epsilon}))$  & $\mathcal{O}(\psi T^\frac{\beta}{\beta+1} \log^{\gamma}(T))$                      
\\ \hline
\end{tabular}
\caption{\small The table shows the resulted communication complexity with different offline algorithms. We use $\tilde{\mathcal{O}}$ to simplify expressions. Key parameters include horizon $T$, number of communicating agents $m$, base arm count $n$, and cardinality constraint $k$. Each row presents a specific offline algorithm or a class transformation with a given offline complexity. For the general rows, we consider classes of offline algorithms, with an approximation error factor $\epsilon$, with general complexity forms, with general constants, $\psi \geq 0$, $\beta \geq 0$, and $\gamma \in \{0,1\}$.  
}
\label{tab:communications}
\end{table*}

\section{Estimation Analysis}
\label{sec:appd:proof:clean-event}

Hoeffding's inequality \cite{hoeffding1994probability} is a powerful technique for bounding probabilities of bounded random variables. We state the inequality, then we use it to show that $\mathcal{E}$ happens with high probability.

\begin{lemma} [Hoeffding's inequality]
\label{lem:hoeffding}
Let $X_1, \cdots, X_n$ be independent random variables bounded in the interval $[0, 1]$, and let $\bar{X}$ denote their empirical mean. Then we have for any $\xi >0$,
\begin{align}
    \mathbb{P}\left( \big|\bar{X} -  \mathbb{E}[\bar{X}] \big| \geq \xi  \right) \leq 2 \mathrm{exp} \left( - 2 n \xi^2  \right). 
\end{align}
\end{lemma}

We use the above Lemma to prove Lemma \ref{lem:probcleanevents} which bounds the probability of the \textit{good event} $\mathcal{E}$ as follows:
\begin{align}
    \mathbb{P}(\mathcal{E}) %
    & \geq 1-\frac{2N(\beta, \gamma, \psi, \epsilon^\star)}{T^2}. \nonumber
\end{align}

\begin{proof}
Applying the Hoeffding bound in \cref{lem:hoeffding} to the empirical mean $\bar{f}(A_i)$ of $A_i$ resulted from $r^\star$ independent rewards and with $\xi=\sqrt{\log(T)/r^\star}$ provides 
\begin{align}
    \mathbb{P} (\bar{\mathcal{E}}_i)&=\mathbb{P}\left[\big|\bar{f}(A_i)-f(A_i) \big| \geq \xi \right] \nonumber\\
    &\leq 2 \mathrm{exp} \left( - 2 r^\star \xi^2  \right) \nonumber\\
    &= 2 \mathrm{exp} \left( - 2 r^\star (\log(T)/ r^\star) \right) \nonumber\\
    &= 2 \mathrm{exp} \left( -  2\log(T)  \right) \nonumber\\
    &= \frac{2}{T^2}. \label{eq:probbnd:single}
\end{align}
Then, we can bound the probability of the \textit{good event} as follows:
\begin{align}
    \mathbb{P}(\mathcal{E}) &= \mathbb{P}(\mathcal{E}_1\cap \dots \cap \mathcal{E}_{N(\beta, \gamma, \psi, \epsilon^\star)}) \nonumber\\
    &=1-\mathbb{P}(\bar{\mathcal{E}}_1\cup \dots \cup \bar{\mathcal{E}}_{N(\beta, \gamma, \psi, \epsilon^\star)}) \tag{De Morgan's Law}\\
    &\geq 1-\sum_{i=1}^{N(\beta, \gamma, \psi, \epsilon^\star)} \mathbb{P}(\bar{\mathcal{E}}_{i})\tag{union bounds} \\
    &\geq 1-\frac{2 N(\beta, \gamma, \psi, \epsilon^\star)}{T^2}.  \tag{using \eqref{eq:probbnd:single}}
\end{align}
\end{proof}

\section{regret Analysis}
\label{appendix:proof:regret}

We establish \cref{thm:main} outlined in \cref{sec:robust} of the main paper. This theorem asserts that for the sequential decision-making scenario delineated in Section \ref{prob_state}, the expected cumulative $\alpha$-regret of the C-MA-MAB, employing an $(\alpha, \beta, \gamma, \psi, \delta)$-resilient-approximation algorithm $\mathcal{A}(\epsilon)$ as a component, is bounded by
$\Tilde{\mathcal{O}}\left(\delta^\frac{2}{3+\beta}\psi^\frac{1}{3+\beta} m^{-\frac{1+\beta}{3+\beta}} T^\frac{2+\beta}{3+\beta}\right)$.

We separate the proof into two cases. One case when the \textit{good event} $\mathcal{E}$  
happens, which we show in \cref{lem:probcleanevents} happens with high probability and then we generalize the result under any event.

\subsection{regret of an Agent under the Good Event}

We upper-bound the expected $\alpha$-regret given the occurrence of the \textit{good event} $\mathcal{E}$. All expectations are conditioned on $\mathcal{E}$ throughout this section. However, for the sake of simplicity in notation, we employ $\mathbb{E}[\cdot]$ rather than $\mathbb{E}[\cdot|\mathcal{E}]$ in certain instances.

We break down the anticipated $\alpha$-regret given $\mathcal{E}$ into two parts: one dealing with regret arising from exploration and the other from exploitation. It is crucial to remember that $f_t(A_t)$ denotes the stochastic reward obtained by selecting action $A_t$, a variable dependent on the historical means of actions in previous rounds.
\begin{align}
    \mathbb{E}[\mathcal{R}_i(T)|\mathcal{E}] 
    & = \sum_{t=1}^T (\alpha f( \mathrm{S^\star} ) - \mathbb{E}[f_t(S_{i,t})])  \nonumber\\
    & = \sum_{t=1}^T (\alpha f( \mathrm{S^\star} ) - \mathbb{E}[\mathbb{E}[ f_t(S_{i,t})  |S_{i,t}  ]])  \\
    & = \sum_{t=1}^T \left(\alpha f( \mathrm{S^\star} ) - \mathbb{E}[f(S_{i,t})]\right) \tag{$f(\cdot)$ defined as expected reward}\\
    & = \sum_{j=1}^{N(\beta, \gamma, \psi, \epsilon^\star)} r_i^\star \left(\alpha f(\mathrm{S^\star})-\mathbb{E}[f(A_j)]\right) + \sum_{t=T_{N(\beta, \gamma, \psi, \epsilon^\star)}+1}^T \left(\alpha f(\mathrm{S^\star})-\mathbb{E}[f(A_t)]\right) %
    \nonumber\\
    &=\underbrace{\sum_{j=1}^{N(\beta, \gamma, \psi, \epsilon^\star)} r_i^\star \left(\alpha f(\mathrm{S^\star})-\mathbb{E}[f(A_j)]\right)}_{\text{Agent $i$ Exploration regret}}+\underbrace{\sum_{t=T_{N(\beta, \gamma, \psi, \epsilon^\star)}+1}^T \left(\alpha f(S^\star)-\mathbb{E}[f(\Theta)]\right)}_{\text{Agent $i$ Exploitation regret}}. \label{eq:prf:mainthm:case1:60}
\end{align}

We establish distinct bounds for the regret stemming from exploration and exploitation individually.

\paragraph{Bounding the agent $i$ exploration regret:}  

\begin{align}
    \sum_{j=1}^{N(\beta, \gamma, \psi, \epsilon^\star)} r_i \left(\alpha f(\mathrm{S^\star})-\mathbb{E}[f(A_i)]\right) &\leq  \sum_{j=1}^{N(\beta, \gamma, \psi, \epsilon^\star)} \frac{r^\star}{m} \left(\alpha-0\right) \tag{rewards in $[0,1]$}\\
    &\leq N(\beta, \gamma, \psi, \epsilon^\star)\frac{r^\star}{m}. \label{eq:cum-regr-explor}
\end{align}

\paragraph{Bounding exploitation regret:}

When the \textit{good event} $\mathcal{E}$ occurs, we know that $|\bar{f}(A)-f(A)|\leq \xi$ for all considered action $A$. For C-MA-MAB framework, with exploitation set $\Theta$, with an $(\alpha, \beta, \gamma, \psi, \delta)$-resilient-approximation $\mathcal{A}(\epsilon)$ as a subroutine, we have after $N(\beta, \gamma, \psi, \epsilon) = \psi \frac{1}{\epsilon^\beta} \log^{\gamma}(\frac{1}{\epsilon})$ oracle calls, for $\epsilon \geq 0$, $\beta \geq 0$, and $\gamma \in \{0,1\}$,
\begin{align}
\mathbb{E}[f(\Theta)]\geq (\alpha - \epsilon) f(\mathrm{S^\star}) -\delta \xi.
\end{align}

Therefore, using $\epsilon^\star$, we have
\begin{align}
\label{eq:ap:final_exp_reward}
\alpha f(S^\star)-\mathbb{E}[f(\Theta)]\leq \delta \xi + \epsilon^\star f(S^\star). 
\end{align}

Therefore, we can bound the exploitation regret as follows:
\begin{align}
    \sum_{t=T_{N(\beta, \gamma, \psi, \epsilon^\star)}+1}^T \left(\alpha f(\mathrm{S^\star})-\mathbb{E}[f(S)]\right) %
    & \leq \sum_{t=T_{N(\beta, \gamma, \psi, \epsilon^\star)}+1}^T (\delta \xi + \epsilon^\star  f(\mathrm{S^\star})) \tag{using \eqref{eq:ap:final_exp_reward}} \nonumber \\
     & \leq \sum_{t=T_{N(\beta, \gamma, \psi, \epsilon^\star)}+1}^T (\delta \xi + \epsilon^\star) \tag{rewards are bounded in $[0,1]$} \nonumber \\
    & \leq T(\delta \xi+\epsilon^\star). \label{eq:cum-regr-exploit}
\end{align}
\paragraph{Bounding total regret:}
\begin{align}
    \mathbb{E}[\mathcal{R}_i(T)|\mathcal{E}] 
    &= \sum_{i=1}^{N(\beta, \gamma, \psi, \epsilon^\star)} \frac{r^\star}{m} \left(\alpha f(\mathrm{S^\star})-\mathbb{E}[f(A_i)]\right)+\sum_{t=T_{N(\beta, \gamma, \psi, \epsilon^\star)}+1}^T \left(\alpha f(\mathrm{S^\star})-\mathbb{E}[f(S)]\right) \tag{copying \eqref{eq:prf:mainthm:case1:60}} \\
    & \leq N(\beta, \gamma, \psi, \epsilon^\star)\frac{r^\star}{m} + T(\delta \xi+\epsilon^\star) \tag{using \eqref{eq:cum-regr-explor}, \eqref{eq:cum-regr-exploit}}
\end{align}
Using the defined confidence radius $\xi = \sqrt{\log(T)/r^\star}$, we have 

\begin{align}
    \mathbb{E}[\mathcal{R}_i(T)|\mathcal{E}] 
    & \leq  N(\beta, \gamma, \psi, \epsilon^\star)\frac{r^\star}{m} + \epsilon^\star T+ T\delta \sqrt{\log(T)/r^\star} \nonumber  \\
    & \leq \psi \frac{1}{{\epsilon^\star}^\beta} \log^{\gamma}(\frac{1}{\epsilon^\star})\frac{r^\star}{m}  + \epsilon^\star T + T\delta \sqrt{\log(T)/r^\star} \label{eq:upperbound_regret}
\end{align}

We note that the inequality in Eq.\eqref{eq:upperbound_regret} is correct for all values of $r^\star \geq m$ and $\epsilon^\star \geq 0$ with the convention that $0^0=1$.

In our algorithm, we choose the following values of $r^\star$ and $\epsilon^\star$ as functions of the problem parameters such as the range $T$ and the number of available agents $m$ as well as on the subroutine algorithm approximation resilience parameter guarantees, that are $\beta$, $\gamma$, $\psi$, and $\delta$. Specifically, using an $(\alpha, \beta, \gamma, \psi, \delta)$-robust approximation as subroutine, we choose $r^\star$ as follows:
\begin{equation}
    r^{\star}= m \left\lceil m^{-1} \left(\delta \sqrt{\log(T)} \left(\frac{Tm}{\psi}\right)^\frac{1}{\beta+1} \right)^{\frac{2+2\beta}{3+\beta}} \right\rceil,
\end{equation}

For clarity in the following steps we further define $z$ as follows:
\begin{equation}
\label{z}
    z= \left(\delta \sqrt{\log(T)} \left(\frac{Tm}{\psi}\right)^\frac{1}{\beta+1} \right)^{\frac{2+2\beta}{3+\beta}} ,
\end{equation}

Moreover, we choose $\epsilon^\star$ as follows:
\begin{equation}
    \epsilon^\star = (\frac{\psi r^\star}{T m})^{\frac{1}{\beta+1}} \mathbf{1}_{\{\beta>0 \text{ OR } \gamma>0\}}.
\end{equation}

Therefore, from Eq. \eqref{eq:upperbound_regret},
\begin{align}
\label{eq:refined_upperbound_regret}
    \mathbb{E}[\mathcal{R}_i(T)|\mathcal{E}] &\leq \psi (\frac{1}{\epsilon^\star})^\beta \log^{\gamma}(\frac{1}{\epsilon^\star})\frac{r^\star}{m}  + \epsilon^\star T + T\delta \sqrt{\log(T)/r^\star} \\
    &\leq \psi (\frac{1}{\epsilon^\star})^\beta \log^{\gamma}(\frac{1}{\epsilon^\star})\left\lceil \frac{z}{m} \right\rceil  + \epsilon^\star T + T\delta \sqrt{\log(T)/(m \left\lceil \frac{z}{m} \right\rceil)} \\
    &\leq \psi (\frac{1}{\epsilon^\star})^\beta \log^{\gamma}(\frac{1}{\epsilon^\star})\left\lceil \frac{z}{m} \right\rceil  + \epsilon^\star T + T\delta \sqrt{\log(T)/ (m \frac{z}{m} )} \nonumber \tag{Since $\lceil z/m \rceil \geq z / m$} \\
    &\leq 2\psi (\frac{1}{\epsilon^\star})^\beta \log^{\gamma}(\frac{1}{\epsilon^\star})\frac{z}{m}  + \epsilon^\star T + T\delta \sqrt{\log(T)/ z } \nonumber \tag{Since $z \geq m$, $\lceil z/m \rceil \leq 2z$}  
\end{align}
where the last inequality follows when $z\geq m$, which holds for $T\geq \frac{\psi m^{\frac{1+\beta}{2}}}{\delta^{\beta+1}}$.

Therefore, 
\begin{equation}
\label{eq:general_regret_boud}
\mathbb{E}[\mathcal{R}_i(T)|\mathcal{E}] \leq 2\psi (\frac{1}{\epsilon^\star})^\beta \log^{\gamma}(\frac{1}{\epsilon^\star})\frac{z}{m}  + \epsilon^\star T + T\delta \sqrt{\log(T)/ z }  
\end{equation}

Recall that $\beta \geq 0$ and $\gamma \in \{0,1\}$. Therefore, we consider all the possible cases (three), the first when $\beta = \gamma = 0$, the second when $\beta =0$ and $\gamma > 0$, and the third when $\beta > 0$.

\textbf{Case 1: $\beta = \gamma = 0$.}

In this case using Eq. \eqref{z}, $z$ becomes the following
\begin{equation}
    z=\left(\delta \sqrt{\log(T)} \left(\frac{Tm}{\psi}\right) \right)^{\frac{2}{3}}.
\end{equation}

Moreover, using Eq.\eqref{epsilon_star}, $\epsilon^\star$ becomes the following
\begin{equation}
    \epsilon^\star = 0.
\end{equation}

From Eq.\eqref{eq:general_regret_boud} we have
\begin{align}
\mathbb{E}[\mathcal{R}_i(T)|\mathcal{E}] &\leq  \frac{2\psi}{m} \left(\delta \sqrt{\log(T)} \left(\frac{Tm}{\psi}\right) \right)^{\frac{2}{3}} + T\delta \sqrt{\log(T)} \left(\delta \sqrt{\log(T)} \left(\frac{Tm}{\psi}\right) \right)^{-\frac{1}{3}}  \\
&\leq  2\frac{\psi}{\psi^{\frac{2}{3}}} \frac{m^{\frac{2}{3}}}{m} \delta^{\frac{2}{3}} \log(T)^{\frac{1}{3}} T^{\frac{2}{3}} + T^{\frac{2}{3}} \delta^{\frac{2}{3}} \log(T)^{\frac{1}{3}} m^{-\frac{1}{3}} \psi^{\frac{1}{3}} \\
&= \mathcal{O}\left(\delta^\frac{2}{3}\psi^\frac{1}{3} m^{-\frac{1}{3}} T^\frac{2}{3}\log(T)^\frac{1}{3}\right). \\
&= \Tilde{\mathcal{O}}\left(\delta^\frac{2}{3}\psi^\frac{1}{3} m^{-\frac{1}{3}} T^\frac{2}{3}\right)  \\
&= \Tilde{\mathcal{O}}\left(\delta^\frac{2}{3+\beta}\psi^\frac{1}{3+\beta} m^{-\frac{1}{3+\beta}} T^\frac{2+\beta}{3+\beta} \right) \tag{Given that $\beta$ = 0}. \nonumber
\end{align}

\textbf{Case 2: $\beta = 0$ and $\gamma \neq 0$.} 

We have $\gamma \neq 0$ and $\gamma \in \{0,1\}$, thus $\gamma = 1$. In this case using Eq. \eqref{z}, $z$ becomes the following
\begin{equation}
    z=\left(\delta \sqrt{\log(T)} \left(\frac{Tm}{\psi}\right) \right)^{\frac{2}{3}}.
\end{equation}

Moreover, using Eq.\eqref{epsilon_star}, $\epsilon^\star$ becomes the following
\begin{equation}
    \epsilon^\star = \frac{\psi r^\star}{T m} .
\end{equation}

From Eq. \eqref{eq:general_regret_boud} we have
\begin{align}
\mathbb{E}[\mathcal{R}_i(T)|\mathcal{E}] &\leq \frac{2\psi}{m} \log(\frac{T m}{\psi r^\star}) z + \frac{\psi r^\star}{T m}  T + T\delta \sqrt{\log(T)} \left(\delta \sqrt{\log(T)} \left(\frac{Tm}{\psi}\right) \right)^{-\frac{1}{3}}  \\
&\leq \frac{2\psi}{m} \log(\frac{T m}{\psi r^\star}) z + \frac{2\psi}{m} z  + T\delta \sqrt{\log(T)} \left(\delta \sqrt{\log(T)} \left(\frac{Tm}{\psi}\right) \right)^{-\frac{1}{3}}  \\
&\leq \frac{2\psi}{m} \log(T) z + \frac{2\psi}{m} z  + T\delta \sqrt{\log(T)} \left(\delta \sqrt{\log(T)} \left(\frac{Tm}{\psi}\right) \right)^{-\frac{1}{3}}  \\
&\leq \frac{4\psi}{m} \log(T) \left(\delta \sqrt{\log(T)} \left(\frac{Tm}{\psi}\right) \right)^{\frac{2}{3}} + T\delta \sqrt{\log(T)} \left(\delta \sqrt{\log(T)} \left(\frac{Tm}{\psi}\right) \right)^{-\frac{1}{3}}  \\
&= \Tilde{\mathcal{O}}\left(\delta^\frac{2}{3}\psi^\frac{1}{3} m^{-\frac{1}{3}} T^\frac{2}{3}\right)  \\
&= \Tilde{\mathcal{O}}\left(\delta^\frac{2}{3+\beta}\psi^\frac{1}{3+\beta} m^{-\frac{1}{3+\beta}} T^\frac{2+\beta}{3+\beta} \right) \tag{Given that $\beta$ = 0}. \nonumber
\end{align}

\textbf{Case 3: $\beta > 0$.}
In this case using Eq. \eqref{z}, $z$ becomes the following
\begin{equation}
    z=\left(\delta \sqrt{\log(T)} \left(\frac{Tm}{\psi}\right)^\frac{1}{\beta+1} \right)^{\frac{2+2\beta}{3+\beta}}.
\end{equation}

Moreover, using Eq.\eqref{epsilon_star}, $\epsilon^\star$ becomes the following
\begin{equation}
    \epsilon^\star = (\frac{\psi r^\star}{T m})^{\frac{1}{\beta+1}} .
\end{equation}

From Eq.\eqref{eq:general_regret_boud} we have
\begin{align}
\mathbb{E}[\mathcal{R}_i(T)|\mathcal{E}] &\leq \frac{2\psi}{m} (\frac{1}{\epsilon^\star})^\beta \log^\gamma(\frac{1}{\epsilon^\star}) z + \epsilon^\star T + T\delta \sqrt{\log(T)} \left( z \right)^{-\frac{1}{2}} 
\end{align}

Therefore, with $\beta>0$ it exists a constant $C$ sufficiently large such as:
\begin{align}
\mathbb{E}[\mathcal{R}_i(T)|\mathcal{E}] &\leq \frac{2 C \psi}{m} (\frac{1}{\epsilon^\star})^\beta z + \epsilon^\star T + T\delta \sqrt{\log(T)} \left( z \right)^{-\frac{1}{2}} \\
&\leq \frac{2 C \psi}{m} (\frac{T m}{\psi r^\star})^{\frac{\beta}{\beta+1}} z + (\frac{\psi r^\star}{T m})^{\frac{1}{\beta+1}} T + T\delta \sqrt{\log(T)} \left( z \right)^{-\frac{1}{2}} \\
&\leq \frac{2 C \psi}{m} (\frac{T m}{\psi z})^{\frac{\beta}{\beta+1}} z + (\frac{2\psi z}{T m})^{\frac{1}{\beta+1}} T + T\delta \sqrt{\log(T)} \left( z \right)^{-\frac{1}{2}} \\
&\leq D \frac{\psi}{m} (\frac{T m}{\psi z})^{\frac{\beta}{\beta+1}} z + T\delta \sqrt{\log(T)} \left( z \right)^{-\frac{1}{2}} \\
&= \Tilde{\mathcal{O}}\left(\delta^\frac{2}{3+\beta}\psi^\frac{1}{3+\beta} m^{-\frac{1}{3+\beta}} T^\frac{2+\beta}{3+\beta} \right). \nonumber
\end{align}

In conclusion, for all the cases, if the \textit{bad event} $\mathcal{E}$ happens, the expected $\alpha$-regret of our C-MA-MAB with an $(\alpha, \beta, \gamma, \psi, \delta)$-resilient-approximation as a subroutine is upper bounded by as follows:
\begin{align}
\mathbb{E}[\mathcal{R}_i(T)|\mathcal{E}] &\leq \Tilde{\mathcal{O}}\left(\delta^\frac{2}{3+\beta}\psi^\frac{1}{3+\beta} m^{-\frac{1}{3+\beta}} T^\frac{2+\beta}{3+\beta} \right). \label{final_upperbound_clean}
\end{align}

\subsection{Regret of an Agent under Any Event}

By \cref{lem:probcleanevents}, the probability of the \textit{bad event} is upper-bounded as follows
\begin{equation}
    \label{proba_bad_event}
    \mathbb{P}(\bar{\mathcal{E}}) \leq \frac{2N(\beta, \gamma, \psi, \epsilon^\star)}{T^2}.
\end{equation}

Since the reward function $f_t(\cdot)$ is upper bounded by 1, the  expected $\alpha$-regret incurred under $\bar{\mathcal{E}}$ for a range $T$ is limited to a maximum of $T$, %
\begin{align}
    \mathbb{E}[\mathcal{R}_i(T)|\bar{\mathcal{E}}] \leq T. \label{eq:badevent:regretbnd}
\end{align}

Combining the results when \textit{bad event} happens and does not happen, we have
\begin{align}
    \mathbb{E}[\mathcal{R}_i(T)] &= \mathbb{E}[\mathcal{R}_i(T)|\mathcal{E}] \cdot \mathbb{P}(\mathcal{E}) +\mathbb{E}[\mathcal{R}_i(T)|\bar{\mathcal{E}}]\cdot \mathbb{P}(\bar{\mathcal{E}}) \tag{Law of total expectation}\\
    &\leq \Tilde{\mathcal{O}}\left(\delta^\frac{2}{3+\beta}\psi^\frac{1}{3+\beta} m^{-\frac{1}{3+\beta}} T^\frac{2+\beta}{3+\beta} \right) 1 +T\cdot \frac{2N(\beta, \gamma, \psi, \epsilon^\star)}{T^2} \tag{using \eqref{final_upperbound_clean}, \eqref{proba_bad_event}, and \eqref{eq:badevent:regretbnd}}\\
    &\leq \Tilde{\mathcal{O}}\left(\delta^\frac{2}{3+\beta}\psi^\frac{1}{3+\beta} m^{-\frac{1}{3+\beta}} T^\frac{2+\beta}{3+\beta} \right) 1 + \frac{2N(\beta, \gamma, \psi, \epsilon^\star)}{T} \\
    &\leq \Tilde{\mathcal{O}}\left(\delta^\frac{2}{3+\beta}\psi^\frac{1}{3+\beta} m^{-\frac{1}{3+\beta}} T^\frac{2+\beta}{3+\beta} \right) 1 + \frac{\Tilde{\mathcal{O}}(\psi^\frac{1}{3+\beta} \psi^{\frac{2+\beta}{3+\beta}} T^{\frac{\beta}{\beta+1}})}{T} \tag{using \cref{lem:queries}} \\
    &= \Tilde{\mathcal{O}}\left(\delta^\frac{2}{3+\beta}\psi^\frac{1}{3+\beta} m^{-\frac{1}{3+\beta}} T^\frac{2+\beta}{3+\beta} \right) + \Tilde{\mathcal{O}}(\psi^\frac{1}{3+\beta} \psi^{\frac{2+\beta}{3+\beta}} T^{-\frac{1}{\beta+1}})  \\
    &= \Tilde{\mathcal{O}}\left(\delta^\frac{2}{3+\beta}\psi^\frac{1}{3+\beta} m^{-\frac{1}{3+\beta}} T^\frac{2+\beta}{3+\beta} \right) + \Tilde{\mathcal{O}}(\psi^\frac{1}{3+\beta} T^{\frac{2+\beta}{3+\beta}} m^{-\frac{2+\beta}{3+\beta}} T^{-\frac{1}{\beta+1}})  \tag{using $T \geq m\psi$} \\
    &= \Tilde{\mathcal{O}}\left(\delta^\frac{2}{3+\beta}\psi^\frac{1}{3+\beta} m^{-\frac{1}{3+\beta}} T^\frac{2+\beta}{3+\beta} \right) + \Tilde{\mathcal{O}}(\psi^\frac{1}{3+\beta} m^{-\frac{1}{3+\beta}-\frac{1+\beta}{3+\beta}} T^{\frac{2+\beta}{3+\beta}-\frac{1}{\beta+1}})  \\
    &= \Tilde{\mathcal{O}}\left(\delta^\frac{2}{3+\beta}\psi^\frac{1}{3+\beta} m^{-\frac{1}{3+\beta}} T^\frac{2+\beta}{3+\beta} \right).
    \nonumber
\end{align}

This concludes the proof.

\newpage
\section{Application to Submodular Maximization}
\label{appendix:application}

We apply our general framework to online stochastic submodular maximization. In the following, we provide examples of submodular maximization problems. We provide results for unconstrained, under cardinality, and under knapsack constraints, by showing the resilience of different offline algorithms in these settings. 

\subsection{Submodular Maximization Examples}

Submodularity arises in critical contexts within combinatorial optimization, such as graph cuts \citep{goemans1995improved, iwata2001combinatorial}, rank functions of matroids \citep{edmonds2003submodular}, and set covering problems \citep{feige1998threshold}. Furthermore, recent works have demonstrated that various real-world problems exhibit submodularity, including data summarization and coreset selection for model training \cite{mirzasoleiman2015lazier, mirzasoleiman2020coresets}, client participation optimization in FL \cite{balakrishnan2022diverse, fourati2023filfl}, recommendation systems \cite{takemori2020submodular}, crowdsourcing, crowdsensing, and influence maximization \cite{fourati2024combinatorial}. 

\subsection{Unconstrained Submodular Maximization}

\begin{lemma}[Generalization of Corollary 2 of \citep{fourati2023randomized}]\label{lem:RandomUSM:robust}
    \textsc{RandomizedUSM} \citep{buchbinder2015tight} is a $(\frac{1}{2},0, 0, 4n, \frac{5}{2}n)$-resilient-approximation algorithm for unconstrained non-monotone SM.
\end{lemma}

\begin{proof}
The offline RandomizedUSM $\frac{1}{2}$-approximation algorithm proposed in \cite{buchbinder2015tight} requires $4n$ oracle calls, thus $\alpha=1/2$, $\psi=4n$, $\gamma=0$, and $\beta=0$. Furthermore, as shown in Corollary 2 of \citep{fourati2023randomized}, using a $\xi$-controlled-estimation $\bar{f}$ of the function $f$,
\begin{equation}
    \begin{aligned}
    \mathbb{E}[f(X_n)]  
     \geq \frac{1}{2} \mathbb{E}[f(OPT)] - \frac{5}{2} n\, \xi.
    \end{aligned}
\end{equation}
Therefore, the RandomizedUSM is an $(\frac{1}{2},0, 0, 4n, \frac{5}{2}n)$-resilient-approximation algorithm for unconstrained non-monotone submodular maximization problem.
\end{proof}

\subsection{Submodular Maximization with Cardinality Constraint}

\begin{lemma}[Generalization of Corollary 4.3 of \citep{nie2022explore}] \label{lem:cardinality:greedy:robust}
    \textsc{Greedy} in \citep{nemhauser1978analysis} is a $(1-\frac{1}{e},0, 0, nk, 2k)$-resilient-approximation algorithm for monotone SM under a cardinality constraint $k$.
\end{lemma}
\begin{proof}
 The offline \textsc{Greedy} $(1-1/e)$-approximation algorithm proposed in \cite{nemhauser1978analysis} requires $nk$ oracle calls, thus $\alpha=(1-1/e)$, $\psi=nk$, $\gamma=0$, and $\beta=0$. Furthermore, as shown in Corollary 4.3 of \citep{nie2022explore}, using a $\xi$-controlled-estimation $\bar{f}$ of the function $f$,
\begin{equation}
    \begin{aligned}
    \mathbb{E}[f(X_n)]  
     \geq (1-1/e) \mathbb{E}[f(OPT)] - 2k\, \xi.
    \end{aligned}
\end{equation}
Therefore, the \textsc{Greedy} is an $(1-\frac{1}{e},0, 0, nk, 2k)$-resilient-approximation algorithm algorithm for monotone submodular maximization under a $k$-cardinality constraint.  
\end{proof}

\begin{lemma}[Generalization of Corollary 1 of \citep{fourati2024combinatorial}] \label{lem:lazier_than_lazy}
    \textsc{Stochastic-Greedy} in \citep{mirzasoleiman2015lazier} is a $(1-\frac{1}{e},0, 1, n, 2k)$-resilient-approximation algorithm for monotone SM under a cardinality constraint $k$.
\end{lemma}

\begin{proof}

The offline \textsc{Stochastic-Greedy} $(1-1/e-\epsilon)$-approximation algorithm proposed in \cite{mirzasoleiman2015lazier} requires $n\log(\frac{1}{\epsilon})$ oracle calls, thus $\alpha=(1-1/e)$, $\psi=n$, $\gamma=1$, and $\beta=0$. Furthermore, as shown in Corollary 1 of \citep{fourati2024combinatorial}, using a $\xi$-controlled-estimation $\bar{f}$ of the function $f$,
\begin{equation}
    \begin{aligned}
    \mathbb{E}[f(X_n)]  
     \geq (1-1/e) \mathbb{E}[f(OPT)] - 2k\, \xi.
    \end{aligned}
\end{equation}
Therefore, the \textsc{Stochastic-Greedy} is an $(1-\frac{1}{e},0, 1, n, 2k)$-resilient-approximation algorithm algorithm for monotone submodular maximization under a $k$-cardinality constraint.

\end{proof}

\begin{lemma} 
\label{lem:random_sampling}
    \textsc{RandomSampling} in \citep{buchbinder2017comparing} is a $(\frac{1}{e},2, 1, n, 4k)$-resilient-approximation algorithm for non-monotone SM under a cardinality constraint $k$.
\end{lemma}

\begin{proof}

The offline \textsc{RandomSampling} $(1/e-\epsilon)$-approximation algorithm proposed in \cite{buchbinder2017comparing} requires $n\frac{1}{\epsilon^2}\log(\frac{1}{\epsilon})$ oracle calls, thus $\alpha=(1/e)$, $\psi=n$, $\gamma=1$, and $\beta=2$.

By design of the algorithm, we have

\begin{align}
f\left(S_i\right)- f\left(S_{i-1}\right) &\geq \bar{f}\left(S_i\right)- \bar{f}\left(S_{i-1}\right) - 2\xi \tag{using $\xi$-controlled estimation} \\
&= 
\max \left\{\bar{f}\left(S_{i-1} \cup \{u_i\}\right) - f\left(S_{i-1}\right), 0\right\} - 2\xi \tag{by algorithm} 
\end{align}

Similar to the steps in \cite{buchbinder2017comparing}, let $A_i$ represent an event encompassing all random decisions made by the algorithm up to iteration $i$ (excluding it). In the initial segment of the proof, we select a specific iteration $1 \leq i \leq k$ and an associated event $A_i$. All probabilities and expectations within this proof portion are implicitly conditioned on $A_i$. It is important to note that, when conditioned on $A_i$, the set $S_{i-1}$ becomes deterministic. We use $v_1, v_2, \ldots, v_k$ to denote the $k$ elements of $\mathcal{N}$ (including $S_{i-1}$) with the highest marginal contribution to $S_{i-1}$, arranged in non-increasing order of marginal contribution. Additionally, let $X_j$ be an indicator for the event $u_i=v_j$.
\begin{align}
\mathbb{E}\left[f\left(S_i\right)- f\left(S_{i-1}\right)\right] & =\mathbb{E}\left[\max \left\{\bar{f}\left(S_{i-1} \cup \{u_i\}\right) - \bar{f}\left(S_{i-1}\right), 0\right\}\right] - 2\xi \\
& \geq \sum_{j=1}^k\left[\mathbb{E}\left[X_j\right] \cdot \max \left\{\bar{f}\left(S_{i-1} \cup \{v_j\}\right) - \bar{f}\left(S_{i-1}\right), 0\right\}\right] - 2\xi\\
& \geq \frac{\sum_{j=1}^k \mathbb{E}\left[X_j\right] \cdot \sum_{j=1}^k \max \left\{\bar{f}\left(S_{i-1} \cup \{v_j\}\right) - \bar{f}\left(S_{i-1}\right), 0\right\}}{k} - 2\xi.
\end{align}

Where the last inequality holds by Chebyshev's sum inequality since $\max \left\{\bar{f}\left(S_{i-1} \cup \{v_j\}\right) - \bar{f}\left(S_{i-1}\right), 0\right\}$ is non-increasing in $j$ by definition and $\mathbb{E}\left[X_j\right]$ is non-increasing in $j$ by Lemma 4.2 in \citep{buchbinder2017comparing}. Therefore,
\begin{align}
\sum_{j=1}^k \max \left\{\bar{f}\left(S_{i-1} \cup \{v_j\}\right) - \bar{f}\left(S_{i-1}\right), 0\right\} & \geq \sum_{u \in O P T} \bar{f}\left(S_{i-1} \cup \{u\}\right) - \bar{f}\left(S_{i-1}\right) \tag{by definition of $v_j$}\\
& \geq \sum_{u \in O P T} (f\left(S_{i-1} \cup \{u\}\right) - f\left(S_{i-1}\right) - 2\xi) \tag{by $\xi$-controlled estimation}\\
& \geq \sum_{u \in O P T} (f\left(S_{i-1} \cup \{u\}\right) - f\left(S_{i-1}\right)) - 2k\xi \tag{$|OPT| \leq k$}\\
& \geq f\left(O P T \cup S_{i-1}\right)-f\left(S_{i-1}\right) - 2k\xi \tag{by submodularity of $f$}.
\end{align}

By Lemma 4.1 in \citep{buchbinder2017comparing}, we have $\mathbb{E}\left[X_j\right] \geq 1-\varepsilon$. Therefore,
\begin{equation}
\mathbb{E}\left[f\left(S_i\right)-f\left(S_{i-1}\right)\right] \geq(1-\varepsilon) \cdot \frac{f\left(O P T \cup S_{i-1}\right)-f\left(S_{i-1}\right)}{k} - 4 \xi.   
\end{equation}

First, since the above equation holds for every given event $A_i$, it also holds in expectation unconditionally. More formally, we get for every $1 \leq i \leq k$,
$$
\mathbb{E}\left[f\left(S_i\right)-f\left(S_{i-1}\right)\right] \geq(1-\varepsilon) \cdot \frac{\mathbb{E}\left[f\left(O P T \cup S_{i-1}\right)\right]-\mathbb{E}\left[f\left(S_{i-1}\right)\right]}{k} - 4 \xi.
$$

Let us lower bound $\mathbb{E}\left[f\left(O P T \cup S_{i-1}\right)\right]$. \textsc{RandomSampling} adds each element to its solution with probability at most $(\lceil p n\rceil / n) / s=1 / k$. Hence, each element belongs to $S_{i-1}$ with probability at most $1-(1-1 / k)^{i-1}$. Let $h(S)=$ $h(S \cup O P T)$. Since $h$ is a nonnegative submodular function, we get by Lemma 2.2,
$$
\mathbb{E}\left[f\left(O P T \cup S_i\right)\right]=\mathbb{E}\left[h\left(S_i\right)\right] \geq(1-1 / k)^i \cdot h(\varnothing)=(1-1 / k)^i \cdot f(O P T) .
$$

Combining the two above inequalities yields,
$$
\begin{aligned}
\mathbb{E}\left[f\left(S_i\right)-f\left(S_{i-1}\right)\right] & \geq(1-\varepsilon) \cdot \frac{(1-1 / k)^{i-1} \cdot f(O P T)-\mathbb{E}\left[f\left(S_{i-1}\right)\right]}{k} - 4 \xi\\
& \geq \frac{\left[(1-1 / k)^{i-1}-\varepsilon\right] \cdot f(O P T)-\mathbb{E}\left[f\left(S_{i-1}\right)\right]}{k} - 4 \xi.
\end{aligned}
$$

We prove by induction the following:
$$
\begin{aligned}
\mathbb{E}\left[f\left(S_i\right)\right] 
& \geq \frac{i}{k} \cdot\left[(1-1 / k)^{i-1}-\varepsilon\right] \cdot f(O P T) - 4i\xi .
\end{aligned}
$$

For $i=0$, the corollary holds since $f\left(S_0\right) \geq 0=$ $(0 / k) \cdot\left[(1-1 / k)^{-1}-\varepsilon\right] \cdot f(O P T)$ - 0. Assume the corollary holds for $i-1 \geq 0$, let us prove it for $i$ :
$$
\begin{aligned}
\mathbb{E}\left[f\left(S_i\right)\right] & \geq \mathbb{E}\left[f\left(S_{i-1}\right)\right]+\frac{\left[(1-1 / k)^{i-1}-\varepsilon\right] \cdot f(O P T)-\mathbb{E}\left[f\left(S_{i-1}\right)\right]}{k} - 4 \xi\\
& =(1-1 / k) \cdot \mathbb{E}\left[f\left(S_{i-1}\right)\right]+\frac{\left[(1-1 / k)^{i-1}-\varepsilon\right] \cdot f(O P T)}{k} - 4 \xi\\
& \geq(1-1 / k) \cdot \left(\frac{i-1}{k} \cdot\left[(1-1 / k)^{i-2}-\varepsilon\right] \cdot f(O P T) -4(i-1)\xi\right)+\frac{\left[(1-1 / k)^{i-1}-\varepsilon\right] \cdot f(O P T)}{k} - 4 \xi \\
& \geq(1-1 / k) \cdot \left(\frac{i-1}{k} \cdot\left[(1-1 / k)^{i-2}-\varepsilon\right] \cdot f(O P T) \right)+\frac{\left[(1-1 / k)^{i-1}-\varepsilon\right] \cdot f(O P T)}{k} -4(i-1)\xi - 4 \xi \\
& \geq \frac{i}{k} \cdot\left[(1-1 / k)^{i-1}-\varepsilon\right] \cdot f(O P T) - 4 i\xi.
\end{aligned}
$$

Plugging $i=k$ into the above corollary yields
$$
\mathbb{E}\left[f\left(S_k\right)\right] \geq\left[(1-1 / k)^{k-1}-\varepsilon\right] \cdot f(O P T) - 4k\xi \geq\left(e^{-1}-\varepsilon\right) \cdot f(O P T) - 4k\xi.
$$

Therefore, the \textsc{RandomSampling} is an $(\frac{1}{e},2, 1, n, 4k)$-resilient-approximation algorithm algorithm for non-monotone submodular maximization under a $k$-cardinality constraint.

\end{proof}

\subsection{Submodular Maximization with Knapsack Constraint}
\label{knapsack_appendix}

We assume that the cost function $c: \Omega \rightarrow R_{>0}$ is both known and linear regarding knapsack constraints. In this framework, the cost linked with a subset is the total of the costs of its elements, denoted as $c(A) = \sum_{v\in A}c(v)$. The \textit{marginal density} is symbolized as $\rho(e|A) = \frac{f(A\cup e)-f(A)}{c(e)}$ for any subset $A\subseteq \Omega$ and element $e\in \Omega\setminus A$. We don't consider scenarios with large budgets $B>\sum_{v\in\Omega} c(v)$ and presume that all items have non-zero costs, meeting $0<c(v)\leq B$. A cardinality constraint represents a particular case with unit costs. Additionally, $q=\frac{B}{c_{\min }}$.

\begin{lemma}[Generalization of Theorem in \cite{nie2023framework}]\label{lem:partial:enum:greedy:robust}
    \textsc{PartialEnumeration} in \citep{sviridenko2004note,khuller1999budgeted} is a $(1-\frac{1}{e},0,0,n^5,4+2\Tilde{K}+2q)$-resilient-approximation algorithm for monotone SM under a knapsack constraint.
\end{lemma}
\begin{proof}
As shown in \cite{nie2023framework}, the \textsc{PartialEnumeration} \citep{sviridenko2004note, khuller1999budgeted} is a $(1-\frac{1}{e},4+2\Tilde{K}+2q)$-robust approximation algorithm for monotone SM under a knapsack constraint. Furthermore, it requires $\mathcal{O}(n^5)$ oracle calls, thus it is a $(1-\frac{1}{e},0,0,n^5,4+2\Tilde{K}+2q)$-resilient-approximation algorithm.
\end{proof}

\begin{corollary}\label{cor:PartialEnumeration:robust}
   C-MA-MAB using the \textsc{PartialEnumeration} in \citep{sviridenko2004note,khuller1999budgeted} as a subroutine requires at most $\tilde{\mathcal{O}}(n^5)$ communication times and yields a $(1-\frac{1}{e})$-regret of at most $\tilde{\mathcal{O}}\left(m^{-\frac{1}{3}} \Tilde{K}^\frac{2}{3} q^\frac{2}{3} n^\frac{5}{3}  T^\frac{2}{3}\right)$ for monotone SM under a knapsack constraint.
\end{corollary}

\begin{lemma}[Generalization of Theorem in \cite{nie2023framework}]\label{lem:greedy:max:robust}
    \textsc{Greedy+Max} in \citep{yaroslavtsev2020bring} is a $(\frac{1}{2},0, 0, n\Tilde{K}, \frac{1}{2}+\Tilde{K}+2q)$-resilient-approximation algorithm for monotone SM problem under a knapsack constraint.
\end{lemma}
\begin{proof}
As shown in \cite{nie2023framework}, \textsc{Greedy+Max} \citep{yaroslavtsev2020bring} is a $(\frac{1}{2},\frac{1}{2}+\Tilde{K}+2q)$-robust approximation algorithm for monotone SM under a knapsack constraint. Furthermore, it requires at most $\mathcal{O}(n\Tilde{K})$ oracle calls, thus it is a $(\frac{1}{2},0, 0, n\Tilde{K}, \frac{1}{2}+\Tilde{K}+2q)$-resilient-approximation algorithm.
\end{proof}

\begin{corollary}
\label{gredyandmax}
 C-MA-MAB using the \textsc{Greedy+Max} in \citep{yaroslavtsev2020bring} as a subroutine requires at most $\tilde{\mathcal{O}}(n\Tilde{K})$ communication times and yields a $(1/2)$-regret of at most $\tilde{\mathcal{O}}\left(m^{-\frac{1}{3}} q^\frac{2}{3} \Tilde{K}  n^\frac{1}{3}  T^\frac{2}{3}\right)$ for monotone SM under a knapsack constraint.
\end{corollary}

\begin{lemma}[Generalization of Theorem in \cite{nie2023framework}]\label{lem:greedy:robust}
    \textsc{Greedy+} in \citep{khuller1999budgeted} is a $(\frac{1}{2}(1-\frac{1}{e}),0, 0, n^2,2+\Tilde{K}+q)$-resilient-approximation algorithm for monotone SM problem under a knapsack constraint.
\end{lemma}
\begin{proof}
As shown in \cite{nie2023framework}, the \textsc{Greedy+} \citep{khuller1999budgeted} is a $(\frac{1}{2}(1-\frac{1}{e}),2+\Tilde{K}+q)$-robust approximation algorithm for monotone SM under a knapsack constraint. Furthermore, it requires at most $n^2$ oracle calls, thus it is a $(\frac{1}{2}(1-\frac{1}{e}),0, 0, n^2 ,2+\Tilde{K}+q)$-resilient-approximation algorithm.
\end{proof}

\begin{corollary}
 C-MA-MAB using the \textsc{Greedy+} in \citep{khuller1999budgeted} as a subroutine requires at most $\tilde{\mathcal{O}}(n^2)$ communication times and yields a $(\frac{1}{2}(1-\frac{1}{e}))$-regret of at most $\tilde{\mathcal{O}}\left(m^{-\frac{1}{3}} q^\frac{2}{3} \Tilde{K}^\frac{2}{3}  n^\frac{2}{3}  T^\frac{2}{3}\right)$ for monotone SM under a knapsack constraint.
\end{corollary}

\section{Experiments on Data Summarization}
\label{additional_experiements}

We conduct experiments on stochastic data summarization. 

\subsection{Motivation for Stochastic Data Summarization}

Data summarization is a primary challenge in machine learning, mainly when dealing with a large dataset. While this problem has been extensively studied in the deterministic case, i.e., with access to a deterministic oracle \cite{mirzasoleiman2015lazier, mirzasoleiman2020coresets, sivasubramanian2024gradient}, this work is the first to address online data summarization under a noisy stochastic objective function. 

\subsubsection{Deterministic Data summarization}
In data summarization, agents must play an action $A$ of at most $k$ images to summarize a large dataset $\mathcal{D}$. Based on a similarity metric $C$ between two images, the deterministic objective is:
\begin{equation}
   \underset{\mathcal{A} \subseteq \mathcal{D}:|\mathcal{A}| \leq k}{\arg \max } \sum_{i \in \mathcal{D} } \max _{v \in \mathcal{A}} C(i,v).  
\end{equation}

Adding more images always increases this objective, and it follows a diminishing return property. Thus, it falls in the monotone SM under a cardinality constraint $k$ \cite{mirzasoleiman2015lazier}. 

\subsubsection{Stochastic Data summarization}

Notice that even evaluating the above deterministic objective for a given action $A$ may become very expensive with a large dataset $\mathcal{D}$; more precisely, the evaluation for one action has a complexity of $\mathcal{O}(|\mathcal{A}||\mathcal{D}|)$. We propose a stochastic version of the above optimization problem and solve it through our framework. We consider a stochastic objective where the chosen subset is compared only to a random subset $\mathcal{R} \subseteq \mathcal{D}$ drawn uniformly at random from the large dataset $\mathcal{D}$, 
%
%
resulting in a noisier but lower complexity evaluations of $\mathcal{O}(|\mathcal{A}||\mathcal{R}|)$. We do not solve the problem for a given realization $\mathcal{R}$, but we solve it in expectation:

\begin{equation}
  \underset{\mathcal{A} \subseteq \mathcal{D}:|\mathcal{A}| \leq k}{\arg \max } \mathbb{E}_{\mathcal{R}} \left[ \sum_{i \in \mathcal{R}} \max _{v \in \mathcal{A}} C(i,v)   \right].   
\end{equation}


\begin{figure}[t]
    \centering
    \includegraphics[width=0.39\textwidth]{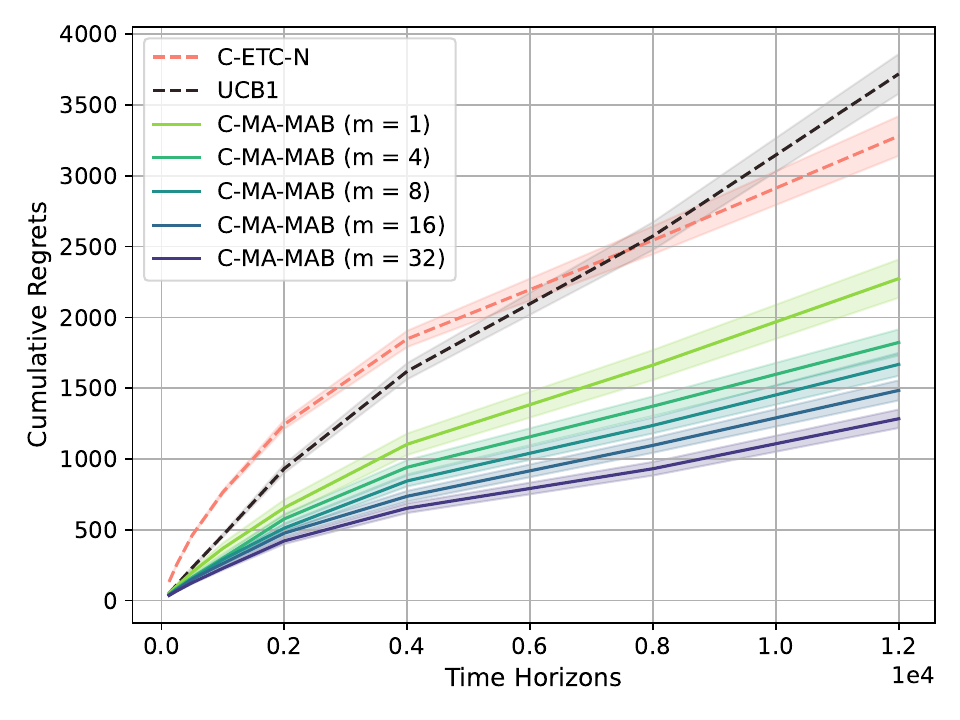}
    \caption{\small Cumulative regrets of summarizing images from FMNIST for different time horizons $T$ using our framework with different number of agents against C-ETC framework and UCB.}
    \label{fig:regrets_famnist}
\end{figure}

\begin{figure}[t]
    \centering
    
    \begin{subfigure}{0.4\textwidth}
        \centering
        \includegraphics[width=\textwidth]{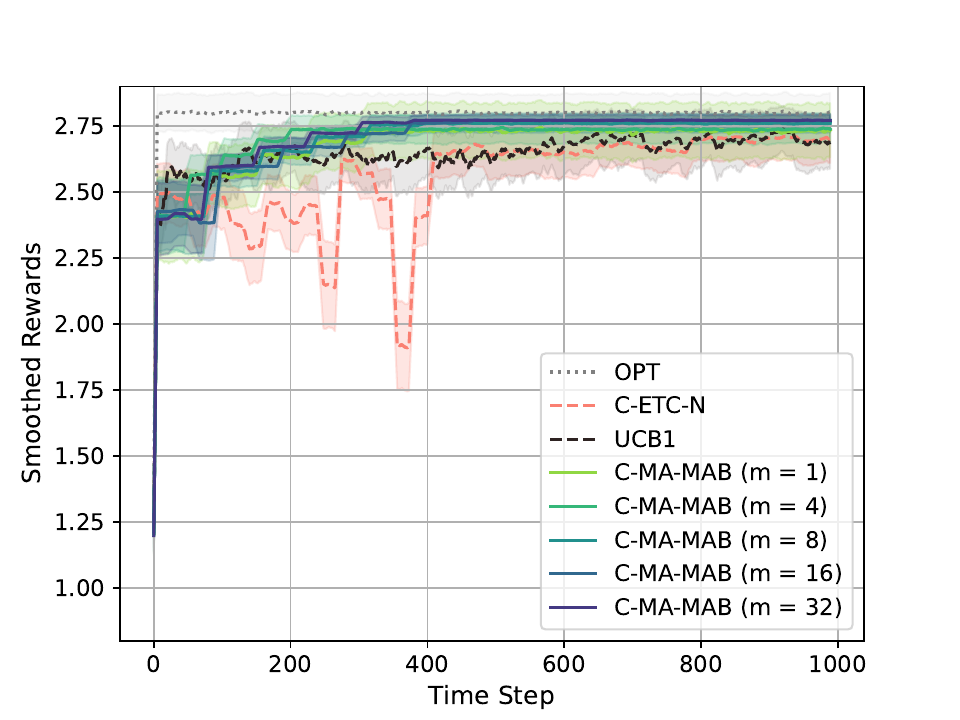}
        \caption{\small CIFAR10}
        \label{fig:cifar-rewards}
    \end{subfigure}
    \begin{subfigure}{0.4\textwidth}
        \centering
        \includegraphics[width=\textwidth]{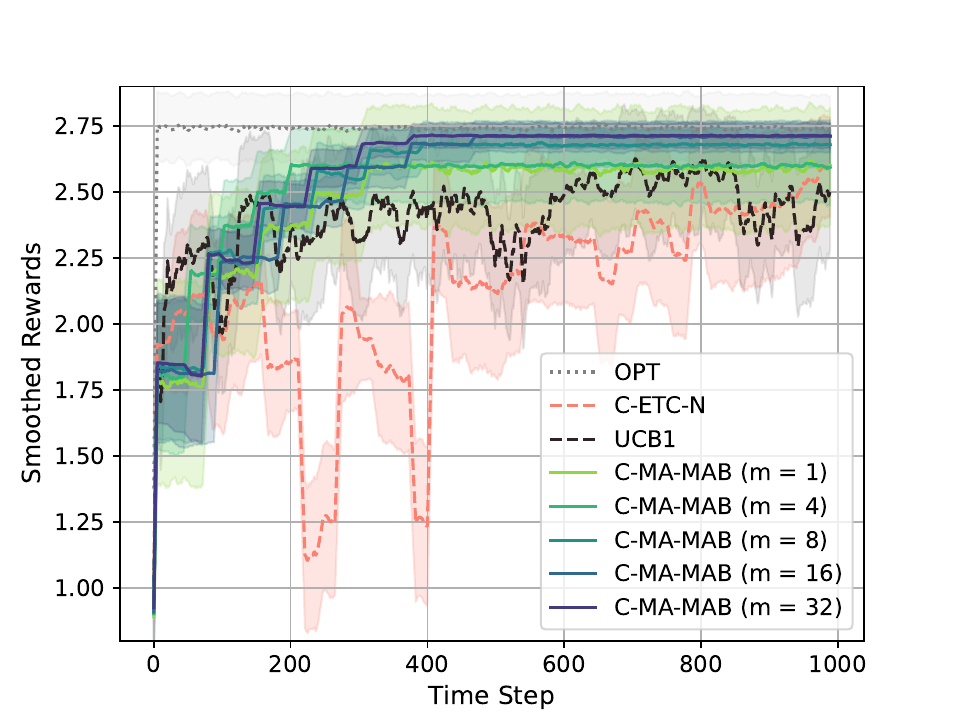}
        \caption{\small FMNIST}
        \label{fig:fmnist-rewards}
    \end{subfigure}
    
    \caption{\small Compare the instantaneous rewards of summarizing images from CIFAR10 and FMNIST for different time horizons $T$ using our framework with one, two, four, sixteen, and thirtytwo agents against C-ETC-N, UCB1, and OPT.}
    \label{fig:combined-rewards}
\end{figure}

\subsection{Experimental Details}

We test our method when using the $\textsc{Stochastic-Greedy}$ algorithm as a subroutine \cite{mirzasoleiman2015lazier} for one agent and multiple agents and compare it to the proposed algorithm in C-ETC framework for SMC (C-ETC-N) \cite{nie2023framework}, and the upper confidence bound (UCB1) algorithm \cite{auer2002finite}. We consider settings with one, four, eight, sixteen, and thirty-two agents. We resize the images to have sixteen pixels. We set a cardinality constraint of $k=5$. Our goal is to summarize information from fifteen images, and instead of comparing it to all the images, we only consider a random batch $\mathcal{R}$ of 3 images. We run the experiments 100 times. We test for several time horizons in $\{125, 250, 500, 1000, 2000, 4000, 8000, 12000, 16000, 20000\}$.

\subsection{Additional Results}

As depicted in Fig. \ref{fig:regrets_cifar10}, the C-MA-MAB demonstrates sub-linear regret guarantees across different agent scenarios within a time horizon $T$. Additionally, it is apparent that, for varying values of $m$, including the single-agent scenario, the C-MA-MAB consistently outperforms both C-ETC-N and UCB1, exhibiting lower regrets over diverse time horizons. Notably, an increase in the number of agents correlates with a reduction in regret for these agents. These observations reinforce the same conclusions drawn from the theoretical analysis.

\end{document}